\documentclass{article}

     \PassOptionsToPackage{numbers, compress}{natbib}
\usepackage{microtype}
\usepackage{graphicx}
\usepackage{subcaption}
\usepackage{booktabs} 
  \usepackage[preprint]{neurips_2026}

\usepackage[utf8]{inputenc} 
\usepackage[T1]{fontenc}    
\usepackage{hyperref}       
\usepackage{url}            
\usepackage{booktabs}       
\usepackage{amsfonts}       
\usepackage{nicefrac}       
\usepackage{microtype}      
\usepackage{xcolor}         

\usepackage{amsmath}
\usepackage{amssymb}
\usepackage{mathtools}
\usepackage{amsthm}
\usepackage{changepage}
\usepackage{rotating}

\usepackage{algorithm}
\usepackage{algpseudocode}

\setcitestyle{numbers}

\theoremstyle{plain}
\newtheorem{theorem}{Theorem}[section]
\newtheorem{proposition}[theorem]{Proposition}
\newtheorem{lemma}[theorem]{Lemma}
\newtheorem{corollary}[theorem]{Corollary}
\theoremstyle{definition}
\newtheorem{definition}[theorem]{Definition}

\newtheorem{remark}[theorem]{Remark}

\title{Inheritance Between Feedforward and Convolutional Networks via Model Projection}

\author{%
  Nicolas Ewen\\
  Department of Mathematics and Statistics\\
  York University\\
  Toronto, Ontario M3J 1P3, Canada \\
   \And
   Jairo Diaz-Rodriguez \\
  Department of Mathematics and Statistics\\
  York University\\
  Toronto, Ontario M3J 1P3, Canada \\
   \AND
   Kelly~Ramsay \\
  Department of Mathematics and Statistics\\
  York University\\
  Toronto, Ontario M3J 1P3, Canada \\
}

\begin{document}

\maketitle

\begin{abstract}
Neural-network techniques are often transferred across architecture families by analogy, but such transfer is valid only when the assumptions required by a technique are preserved. We 
introduce this idea as \emph{inheritance} between model classes. Using a unified node-level framework with tensor-valued activations, we prove that generalized feedforward networks (GFFNs) form a strict subset of generalized convolutional networks (GCNNs), so GCNN 
properties transfer directly to GFFNs. The reverse direction is not automatic: standard CNN nodes use spatial kernels, while FFN nodes use one scalar weight per input contribution. We introduce \emph{model projection} to recover a restricted reverse inheritance path. Projection freezes each convolutional input-channel sub-function and learns one scalar coefficient for each input-output channel contribution, giving projected CNN nodes the GFFN-style trainable structure of scalar-weighted input recombination. This inherited structure leads naturally to parameter-efficient transfer learning. Across multiple ImageNet-pretrained CNN backbones and downstream image-classification datasets, model projection is competitive with standard and PEFT baselines and provides an effective initialization for subsequent full fine-tuning.
\end{abstract}


\section{Introduction}

Techniques and intuitions developed for one neural-network family are often reused in another, but this transfer is rarely made precise. FFNs and CNNs are a common example: ideas developed for one are frequently applied to the other, even though the two families are usually formalized with different input structures and node operations. We study this issue through the lens of \emph{inheritance}: the transfer of a property between model classes through an explicit structural relationship that preserves the assumptions required by that property. This distinguishes valid transfer from architectural analogy. Rather than asking whether two architectures look similar, we ask which assumptions are preserved when one model class is related to another, and which properties therefore transfer.

This perspective is especially useful for FFNs and CNNs because the two classes differ in their trainable input contributions. A traditional FFN node has one scalar weight per input contribution, whereas a standard CNN node has a spatial kernel for each input-output channel contribution. Thus, techniques whose assumptions depend on a linear combination in which each input contribution is controlled by a single scalar weight apply naturally to FFN nodes, but not automatically to ordinary CNN nodes. If a structural mapping allows CNNs to inherit this scalar-weighted trainable structure from FFNs, then the inherited structure can be naturally exploited for parameter-efficient transfer learning. This motivates our central question: \emph{when can CNNs and FFNs inherit properties from one another, and can such inheritance be used for parameter-efficient transfer learning?}

\textbf{A unified formalization.}
We first introduce a unified node-level formalization with tensor-valued activations. In this framework, generalized feedforward network (GFFN) nodes perform scalar-weighted summation across input channels while preserving tensor structure, whereas generalized convolutional network (GCNN) nodes apply channel-wise convolutional transformations and then sum across channels. This lets us compare the two model classes directly. We prove that GFFNs form a strict subset of GCNNs: a convolution with kernel size one in every spatial dimension reduces to scalar multiplication of each input channel. Therefore, any property that holds for all GCNNs also holds for the corresponding GFFN subclass.

\textbf{Model projection.}
The reverse direction does not hold by subclass inclusion. Arbitrary CNNs are not GFFNs, because their input-output channel contributions are spatial kernels rather than scalar weights. To recover a restricted reverse direction, we introduce \emph{model projection}. Projection uses the fact that CNN node functions are separable by input channel before the nonlinearity. It freezes each convolutional input-channel sub-function and learns one scalar coefficient for each input-output channel contribution. The resulting projected node has the same scalar-weighted trainable summation structure as a GFFN node, possibly with inhomogeneous inputs.

\textbf{PEFT benefit.}
This gives the precise inheritance statement needed for our application: projected CNN nodes inherit the FFN-style trainable structure of one scalar trainable weight per input contribution. When applied to convolutional layers, this inherited property directly yields a parameter-efficient transfer-learning method. Instead of adapting an input-output channel contribution by training an entire spatial kernel, model projection freezes the pretrained kernel and trains only a scalar coefficient. Model projection therefore preserves pretrained spatial feature extractors while allowing every convolutional layer to adapt through scalar channel-contribution weights rather than full spatial kernels
(See Fig.~\ref{fig:projExplainer}).

\textbf{Empirical evaluation.}
We evaluate model projection as a parameter-efficient transfer-learning method for ImageNet-pretrained CNN backbones. Across multiple downstream image classification datasets and architectures, projection is competitive with logistic regression, full fine-tuning, and established PEFT baselines. Empirically, projection often combines the stability of partial tuning with the continued improvement of adapting the full model. In two-stage training, projection also provides a strong initialization for subsequent full fine-tuning.

\textbf{Contributions.}
Our main contributions are:
\begin{itemize}
    \item We introduce unified tensor-compatible formalizations of FFN and CNN nodes, enabling direct comparison between the corresponding model classes. Within this framework, we prove that GFFNs form a strict subset of GCNNs, establishing one rigorous inheritance direction.

    \item We introduce model projection as a structural mapping from separable CNN nodes to projected nodes with GFFN-style scalar-weighted trainable summation. We prove that projected nodes inherit FFN properties and techniques whose assumptions depend only on this structure and do not require homogeneous inputs across nodes.
    
    \item We turn this inherited scalar-weighted structure into a parameter-efficient transfer-learning method for CNNs. Projection adapts convolutional layers by training one scalar per input-output channel contribution rather than full spatial kernels, and experiments show that this yields a competitive accuracy-efficiency trade-off across transfer-learning benchmarks.

    \item We provide an open-source implementation of projection and demonstrate its practical value for transfer learning across multiple public datasets, showing that projection can adapt entire models while training substantially fewer parameters.
    
\end{itemize}

\subsection{Related Work}

\paragraph{Formalizations and inheritance across model families.}
Formal treatments of neural-network computations are often specialized to particular architectures and input shapes rather than expressed in a unified, tensor-level language. Classical discussions of feedforward networks typically assume vector-valued scalar inputs \citep{hastieESL}, while standard formulations of convolutional networks are usually presented for specific spatial cases, most commonly two-dimensional inputs \citep{lecun2010convolutional}. This separation makes it harder to state precise relationships between model classes, and therefore to identify when results or techniques from one class preserve the assumptions needed to apply to another. Although ideas are frequently transferred informally between FFNs and CNNs, explicit structural accounts of such transfer appear limited. This motivates a framework that makes the preserved node-level structure explicit.

\paragraph{Parameter-efficient transfer learning.}
Transfer learning adapts a pretrained model to a downstream task, often in low-data regimes where full fine-tuning can overfit or be difficult to optimize \citep{chollet2021deep,han2024facing,jia2022visual,wang2023real}. Parameter-efficient transfer learning (PEFT) restricts the trainable degrees of freedom \citep{kornblith2019better}. BatchNorm adaptation updates affine normalization parameters \citep{ioffe2015batch}, but these act after convolutional input contributions have already been summed; projection instead rescales each fixed input-output channel contribution before aggregation. BitFit tunes bias terms \citep{zaken2022bitfit}, but biases shift node outputs rather than selectively reweighting individual contributions. Scale-shift methods modulate features with lightweight affine transformations \citep{lian2022scaling}, but typically act at the activation or channel level rather than per input-output contribution. Residual and convolutional adapters add auxiliary trainable modules \citep{rebuffi2017learning,chen2024conv,luo2023towards}, whereas projection keeps the original convolutional structure and only rescales pretrained filter-channel contributions. Low-rank adaptation learns additive weight updates \citep{hu2022lora}, including convolution-specific variants \citep{ding2024lora}, while projection restricts each pretrained filter channel to a one-dimensional multiplicative subspace. Sensitivity-aware methods allocate limited tuning capacity across layers \citep{he2023sensitivity}, and are complementary to projection. Thus, projection is more fine-grained than input- or output-channel scaling, but more constrained than adapters or low-rank updates. 

\paragraph{Pre-defined CNNs.}
Model projection also differs from pre-defined CNNs, where a set of fixed filters produce outputs that are then passed to a layer of $1 \times 1$ convolutional nodes \citep{linse2023convolutional}. In projection, the pretrained convolutional filter channels are not first aggregated into a shared intermediate representation. Instead, each output node rescales its own fixed input-channel contributions before summation. We elaborate on this distinction in Remark~\ref{pre-definedCNN}.

\section{Background and Problem}

\subsection{Problem}

Our goal is to answer the following questions: \textit{Under what conditions, if any, do CNNs naturally inherit results from FFNs, and vice versa}, and, \textit{does this inheritance aid in transfer learning?}
For example, traditional feed forward neural networks train with fewer parameters per node. When fine tuning a CNN foundation model for a downstream task, can we inherit this property to efficiently train an entire neural network with fewer parameters?


Before introducing the relevant background, we first formalize what is meant by a neural network model. 
We define a network model $M\coloneqq M(\theta)$, where $\theta\in\mathbb{R}^p$, to be some function composed of $L\in\mathbb{N}$ sequential layers. That is, $M$ can be expressed as a composition of $L$ subfunctions or \emph{layers} $M(\theta)=f_L \circ \ldots  \circ f_1$. Given input $Z_i$ and parameters $W_i$, each layer $f_i$ produces a sequence of at least two outputs $a_{ij},\ j=1,\ldots,J_i, J_i \geq 2$. Each output $a_{ij}$ corresponds to a node $f_{ij}$ and subset of parameters $W_{ij}$. We may refer to $a_{ij}$ as the $j$th output \emph{channel} of layer $i$. 
Precisely, each layer can be written as
\begin{align*}
   f_i(Z_i, W_i) &= [f_{i1}(Z_i, W_{i1}),\ldots, f_{iJ_i}(Z_i, W_{iJ_i})] = [a_{i1}, \ldots, a_{iJ_i}] \\&= [Z_{(i+1,1)}, \dots , Z_{(i+1,J_i)}] = Z_{i+1}.
\end{align*}


\subsection{Traditional Neural Networks}
Traditional neural networks (FFNs) are a type of network model, where data flows along edges, and calculations are done in nodes. We summarize the description from Hastie et al. \cite{hastieESL}. Traditionally, these calculations are linear regressions, followed by a non-linear function $\sigma$. 
Let subset $W_{jk}$ be the subset of parameters $W_j$ that are associated with input channel $Z_k$.
Given an input $Z \in \mathbb{R}^{d}$, a parameter set $W_j$ containing a vector of weights $W_{jk} \in \mathbb{R}^{d}$, and a bias $b_{j} \in \mathbb{R}$, the node will have the following form:
\begin{equation}\label{FFN}
    f_{j}(Z, W_{j}) 
    = \sigma(\sum_{k=1}^{d}Z_{k}W_{jk} + b_{j}).
\end{equation}
This formulation requires that the input and output data are vectors of scalars. While more general versions have been implemented in code \cite{chollet2021deep}, to the best of our knowledge, all formal treatments have so far required non-scalar data to be vectorized first.

\subsection{Convolutional Neural Networks}
Convolutional neural networks (CNNs) are another type of network model, but where some nodes perform a convolution. For simplicity, we describe the 2-dimensional case from \citet{lecun2010convolutional}, from which the more general case extends naturally. Given an input $Z \in \mathbb{R}^{d \times n_1 \times n_2}$, and a parameter set containing a filter $F_{j} \in \mathbb{R}^{d \times l_1 \times l_2}$ with filter channels $F_{jk} \in \mathbb{R}^{l_1 \times l_2}$, and a bias $B_{j} \in \mathbb{R}^{m_1 \times m_2}$, a convolutional node $f_{j}$ will have the following form, where $\ast$ denotes the convolution operation:
\begin{equation}\label{CNN}
    f_{j}(Z, W_{j}) = \sigma(\sum_k F_{jk} \ast Z_{k} + B_{j}).
\end{equation}
$F_{jk}$ connects input channel $Z_{k}$ to output $a_{j}$. We consider the node $f_{j}$ to be the collection of operations and parameters connecting $Z$ to output $a_{j}$, plus the bias.
While CNNs often operate on 1-, 2-, or even 3-dimensional data, the common formulations often only specify the 2-dimensional case. 
For further reading see \citep{lecun2010convolutional, chollet2021deep, goodfellow2016deep}.


\section{Inheritance between CNNs and FFNs}\label{sec:proofs}
\subsection{Node Generalization and Model Subclasses}
We can now address our earlier question of inheritance through \emph{model subclasses} and \emph{model projections}. We start by generalizing equations for commonly used neural network models.
In order to do so, we define a generalization of the dot product that outputs a tensor called the tensor dot product. 
Let a tensor be a multidimensional array with shape $D$. 
Let $\Gamma_D$ be a tensor of shape $D$ where all values are $1$.
\begin{definition}
For $Z \in \mathbb{R}^{d \times D}$, and $W_{jk} \in \mathbb{R}^{d}$ the simplified tensor dot product, denoted $\odot_t$, is $Z \odot_t W_{jk} = \sum_{k=1}^{d} Z_{k}W_{jk}$. 
\end{definition}
The simplified tensor dot product takes a vector $Z$, of $d$ tensors, $Z_{k} \in \mathbb{R}^{D}$, of uniform size and shape $D$, as well as a vector of $d$ scalars, $W_{jk}$, and outputs a tensor of size and shape $D$. Here, $Z_{k}W_{jk}$ is scalar multiplication. If the size of $Z_{k}$ is $1$ in all dimensions, $\odot_t$ reduces to the traditional dot product. 

\begin{remark}
For ease of exposition, we introduced a simpler version of the tensor dot product, namely $W_{jk}$ is assumed to be a vector of scalars, whereas commonly used implementations use a more general tensor dot product for such nodes, such as in dense layers in Tensorflow \cite{tensorflow2015-whitepaper, chollet2021deep}, and linear layers in Pytorch \cite{Ansel_PyTorch_2_Faster_2024}.
\end{remark}

We now generalize FFN nodes to apply to structures of data other than vectors of scalars, and define a general convolutional node for a tensor of shape $D$. Let $\Gamma_D$ be a tensor of shape $D$ where all values are equal to one. 
\begin{definition}\label{gFFN}
For $Z \in \mathbb{R}^{d \times D_{in}}$, $b_{j} \in \mathbb{R}, \Gamma_{D_{out}} \in \mathbb{R}^{D_{out}}$, and $W_{jk} \in \mathbb{R}^{d}$
 the GFFN node is:
\begin{equation}\label{gffnEqn}
\begin{split}
f_{j}(Z, W_{j}) = \sigma(Z \odot_t W_{jk} + b_{j}\Gamma_{D_{out}}). = \sigma(\sum_{k=1}^{d} Z_{k}  W_{jk} + b_{j}\Gamma_{D_{out}}).
\end{split}
\end{equation}
\end{definition}
To generalize an FFN node from Equation \ref{FFN}, we use the tensor dot product and replace 
$b_{j} \in \mathbb{R}$ with $b_{j}\Gamma_{D_{out}}
\in \mathbb{R}^{D_{out}}$. 
In other words, 
the bias is now a tensor with all values being $b_{j}$. 
%
 Traditional neural network nodes are a special case of GFFN nodes where the input data is a vector of scalars. If the 3-dimensional input is a sequence of 2D image channels, then a GFFN node acts as a $1 \times 1$ 
 CNN node. To act as a fully connected layer in the traditional sense, the input can be flattened first.  
We now define GCNN nodes. 
\begin{definition}\label{gCNN}
For $F_{j} \in \mathbb{R}^{d \times D_{F_{j}}}$, $Z \in \mathbb{R}^{d \times D_{in}}$, and $b_{j}\Gamma_{D_{out}} \in \mathbb{R}^{D_{out}}$ 
 the GCNN node is:
 \begin{equation}\label{gCNNEqn}
  \begin{split}
      f_{j}(Z, W_{j}) =
\sigma(Z \ast F_{j} + b_{j}\Gamma_{D_{out}}). = \sigma(\sum_{k=1}^{d} Z_{k} \ast F_{jk} + b_{j}\Gamma_{D_{out}}).
  \end{split}   
 \end{equation}
\end{definition}
That is, to generalize a CNN node from Equation \ref{CNN}, we replace
$F_{j} \in \mathbb{R}^{d \times l_1 \times l_2}$ with $F_{j} \in \mathbb{R}^{d \times D_{F_{j}}}$, we replace
$Z \in \mathbb{R}^{d \times n_1 \times n_2}$ with $Z \in \mathbb{R}^{d \times D_{in}}$,
and we replace $B_{j} \in \mathbb{R}^{m_1 \times m_2}$ with $b_{j}\Gamma_{D_{out}} \in \mathbb{R}^{D_{out}}$.
Colloquially, the filters and data have the same number of dimensions in each channel, but the number of dimensions is not specified globally for the network. Instead of a vector of 2D images, we now have a vector of tensors. The bias is now a tensor with all values being $b_{j}$. 

Now that we have formalized general definitions, we can prove that the class of traditional feed forward neural networks 
(FFNs) is a subclass of the class of convolutional neural networks (CNNs). 
\begin{theorem}\label{subsetsMain}
GFFNs are a strict subset of GCNNs.
\end{theorem}
The proof of this result can be found in Appendix~\ref{supProofs}.
To prove Theorem~\ref{subsetsMain}, we establish a bijection between the set of GFFNs and a subset of GCNNs, namely the set of GCNNs with kernels of size one in all dimensions. 
An important consequence of Theorem~\ref{subsetsMain} is the following.
\begin{corollary}\label{inheritanceMain}
Since the class of GFFNs is contained in the class of GCNNs, any property that holds for all GCNNs also holds for the corresponding subclass of GFFNs.
\end{corollary}
Corollary~\ref{inheritanceMain} says that any results or techniques that apply to all GCNNs automatically apply to all GFFNs. Thus, any techniques and results from GCCNs carry over to GFFNs. 
\subsection{Model Projection}
Of course, CNNs are not a subclass of FNNs. As a consequence, we now introduce \emph{model projection}: a method to project CNNs and other non-FFN nodes, onto FFN nodes. 
This allows inheritance of many properties exclusive to FFNs, such as training a node with one weight per input.
We first define a node function.
\begin{definition}
The \emph{node function} $f_j'$ associated with a node $f_j$ is defined by
\[
f_j(Z,W_j) = \sigma\big(f_j'(Z,W_j)\big),
\]
that is, $f_j'$ is the pre-activation (inner) function of the node.
\end{definition}
We now define \textit{separable by input}, a property that a node function is required to have in order for the node to be projected. 
\begin{definition}
A node function $f_{j}'$ is separable by input if, for an input $Z$ with $d$ elements, $f_{j}'$ takes the following form: $f_{j}'(Z, W_{j}) =  \sum_{k=1}^{d} f_{jk}'(Z_{k}, W_{jk}) + b_{j}\Gamma_{D_{out}}$, where each $f_{jk}'$ depends only on the $k$th input component $Z_k$ and a corresponding subset $W_{jk} \subset W_j$ of parameters. 
\end{definition}
``Separable by input'' means that the node function decomposes into a sum of node sub-functions $f_{jk}'$ plus a bias, where each $f_{jk}'$ depends on only one input channel.  
For example, the function $g_1(x_1,x_2) = a_1x_1^2 + a_2x_2^2 +a_3x_1 + a_4x_2$, where $a_1,\ldots a_4\in\mathbb{R}$ is separable by input, since we can write 
\begin{align*}
    g_1(x_1,x_2) = (a_1x_1^2 + a_3x_1) + (a_2x_2^2 + a_4x_2)
    = g_{11}(x_1) + g_{12}(x_2). 
\end{align*}
On the other hand, the function $g_2(x_1,x_2) =  a_1x_1^2+ a_2x_2^2 + a_3x_1 + a_4x_2 + x_1x_2$ is not separable by input, because of the $x_1 x_2$ term. 
GCNN node functions are separable by input, see Lemma \ref{separablCNNs} in the Appendix. Node functions that are separable by input can be ``projected.''

\begin{definition}\label{def_proj}
Given a node function that is separable by input, its \emph{projected node} is defined as $\hat{f}_{j}(Z,W_{j}, \gamma_{j}) = \sigma(\sum_{k=1}^{d} \gamma_{jk} f_{jk}'(Z_{k},W_{jk}) + b_{j}\Gamma_{D_{out}})$. 
A model's \emph{GFFN projection} is the model that results from projecting all non-GFFN nodes that are separable by input and leaving the rest of the network unchanged. 
\end{definition}


Here, the idea is that if we view the weights $W_{jk}$ as fixed, and consider training the new weights $\gamma_{jk}$ for a downstream task, then the projected model is essentially a GFFN with node-specific pre-processing and inherits all properties, including training techniques, for a GFFN that do not rely on homogeneous inputs, which are discussed next. 

We now introduce a property of a layer called \textit{homogeneous inputs}.
\begin{definition}
A layer has \textit{homogeneous inputs} if all nodes in the layer have the same input, otherwise it has inhomogeneous inputs. A node has homogeneous inputs if it is part of a layer with homogeneous inputs, otherwise it has inhomogeneous inputs. 
\end{definition}
Our next result says that projected nodes are GFFN nodes with inhomogeneous inputs.
\begin{theorem}\label{projectionToGFFN}
Projected nodes are GFFN nodes with inhomogeneous inputs.
\end{theorem}
\begin{proof}
Let $\hat{Z}_{jk}$ denote the output of the node sub-function with fixed weights $f_{jk}'(Z_{k}, W_{jk})$. Then we can rewrite the equation for a projected node as: 
\begin{align*}
    \hat{f}_{j}(Z,W_{j}, \gamma_{j}) &= \sigma(\sum_{k=1}^{d} \gamma_{jk} f_{jk}'(Z_{k},W_{jk}) + b_{j}\Gamma_{D_{out}})\\
    &= \sigma(\sum_{k=1}^{d} \gamma_{jk} \hat{Z}_{jk} + b_{j}\Gamma_{D_{out}}) \\&= \sigma(\hat{Z}_{j} \odot_t \gamma_{j} + b_{j}\Gamma_{D_{out}}).
\end{align*}
The final expression above is in the form of Equation \ref{gffnEqn}, with $\hat{Z}_{j}$ replacing $Z_{j}$, and $\gamma_{j}$ replacing $W_{j}$. 
\end{proof}
A direct consequence of Theorem~\ref{projectionToGFFN} is the following. 
\begin{corollary}\label{projection}
Since projected nodes are GFFN nodes with inhomogeneous inputs, properties that apply to GFFN nodes with inhomogeneous inputs also apply to projected nodes. 
\end{corollary}
Corollary~\ref{projection} says that projected layers inherit properties that apply to GFFN layers with inhomogeneous inputs. 
In particular, given that GCNN node functions are separable by input, see Lemma \ref{separablCNNs} in the Appendix, we can project GCNN layers. 
This is especially useful for transfer learning. 
For example, GFFN nodes can be trained with one weight per input channel. This is true even in layers with inhomogeneous inputs. Therefore, a projected GCNN node can be trained with one weight per input channel, potentially greatly reducing the total number of parameters being trained.

\begin{figure*}t]
  \centering
  \includegraphics[width=\textwidth]{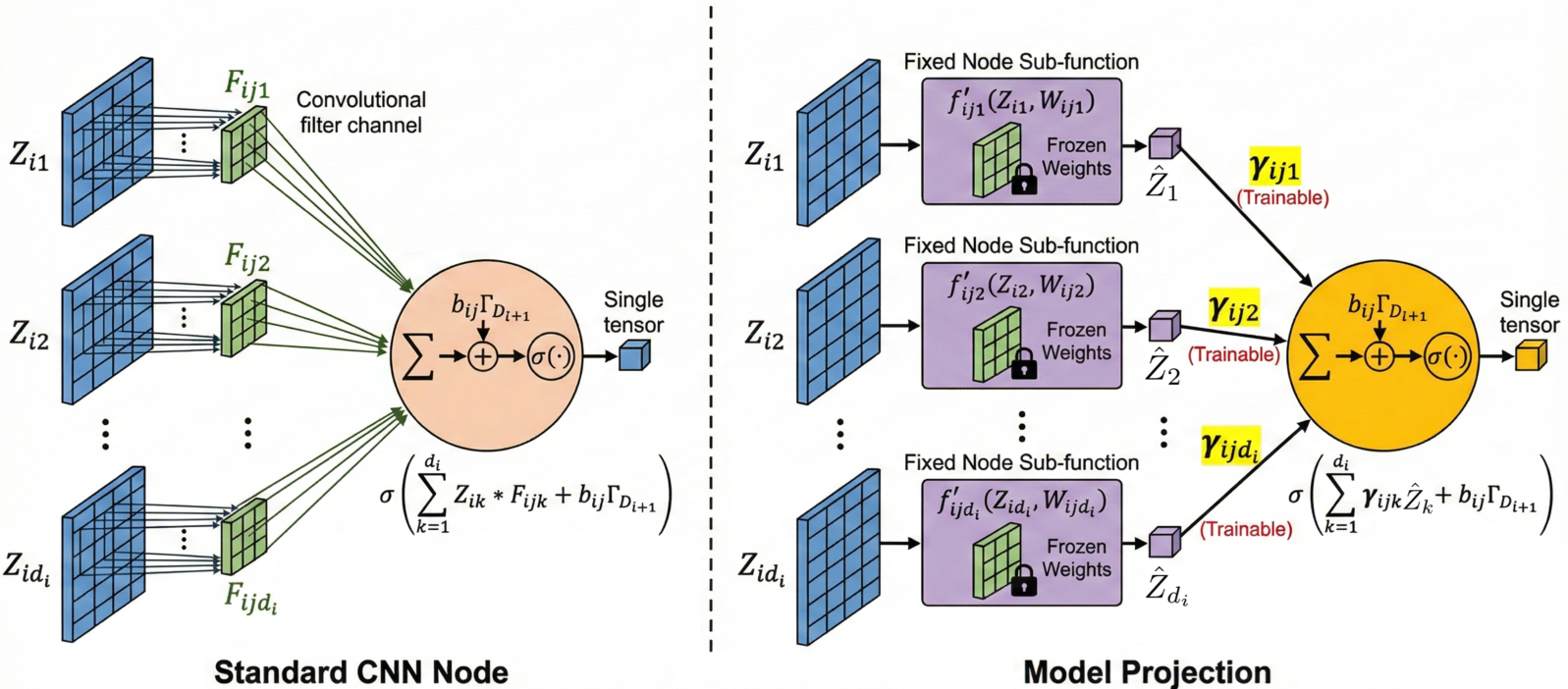}
  \caption{ (left) Standard CNN Node: Each input channel contributes through multiple learned weights due to the spatial extent of the kernels. 
(right) Model projection: The node has exactly one trainable weight per input channel. Spatial structure is preserved by fixed sub-functions, while channel interaction is reduced to a linear weighted sum. The resulting computation is structurally identical to an FFN node with processed tensors as inputs, suitable for PEFT.}
  \label{fig:projExplainer}
  \vspace{-0.5cm}
\end{figure*}

In order to test the usefulness of model projection for transfer learning, we implemented model projection for 2D convolutional layers as a modified version of a 2D convolutional layer. The main difference is that it also includes a new set of parameters, $\gamma_{jk}$, one for each channel, $k$, in each convolutional node, $j$. We also fix the filter weights, while the new parameters and biases are trained.
Each $\gamma_{jk}$ is initialized to one, so that the first forward pass through the network will be the same as before model projection.
Projected nodes learn a linear regression before $\sigma$ is applied, but unlike a GFFN node, projected nodes also perform an extra pre-processing step to their input. 
Intuitively, our specific design choice is intended to implement Theorem 3.11, allowing us to turn a CNN into a kind of FFN. While other freezing techniques can reduce total trainable parameters, ours keeps the trainable part in a FFN node design, which is already known to be successful. This should not only allow us to benefit from our demonstrated property, but also benefit from other strengths of FFNs.
As stated earlier, one advantage of projecting nodes is that it allows a node to be trained with fewer parameters, reducing the amount of data needed. 
We explore this idea in Section~\ref{sec:ex}.



\section{Experiments}\label{sec:ex}
We now demonstrate the utility of inheritance via projection for transfer learning: 
training nodes with only one parameter per input channel. We focus on transfer learning large pre-trained foundation CNN models onto downstream tasks. Normally, CNN nodes can have many parameters per input channel. By projecting these models, we can train an entire model using many fewer parameters, which should be advantageous when 
there is not enough data to fully train these large networks.

We conduct two sets of experiments. The first compares model projection with full fine-tuning and logistic regression 
, with an emphasis on analyzing training behavior. The second evaluates model projection on standard benchmarks and compares its performance to state-of-the-art 
PEFT methods.

\subsection{First Experiment: Understanding Model Projection Behavior}

We conduct transfer learning experiments on seven benchmark datasets: Food-101 \cite{bossard14}, CIFAR-10 \cite{Krizhevsky09learningmultiple}, CIFAR-100 \cite{Krizhevsky09learningmultiple}, Stanford Dogs \cite{stanford_dogs_FGVC2011}, Oxford-IIIT Pets \cite{parkhi12a}, Caltech-101 \cite{FeiFei2004LearningGV}, and Oxford 102 Flowers \cite{Nilsback08}. These datasets range in size from about 2,000 to about 75,000 training examples and span both coarse-grained and fine-grained image classification tasks. See Appendix Table~\ref{tab:datasets_lines}.
Our focus is on showing a proof of concept of our method rather than pushing state-of-the-art results, so we intentionally use simple training setups without hyperparameter optimization.


\textbf{General setup.} We perform experiments using three different convolutional bases, with weights pre-trained on ImageNet\cite{deng2009imagenet}: VGG16\cite{simonyan2015very}, ResNet50\cite{he2016deep}, and DenseNet121\cite{huang2017densely}. After each convolutional base, we add a global average pooling layer, a dropout layer set to 0.5, and finally a single dense layer for predictions. We use simple data augmentations: horizontal flip with probability 0.5, and random rotation and zoom both set to 0.1. Input images were set to 224 x 224 pixels.
We do not shuffle data or address class imbalances. Nor do we address noisy labels or perform any other type of regularization. Any validation sets are absorbed into training sets. 
We evaluate two training setups:

\paragraph{Single-stage training.}
For each model, we run three variants:
Logistic regression (LR), where 
only the final classifier is trained;
Full fine-tuning (FT), where all parameters are trained; and
model projection, where convolutional kernels are frozen and 
 all other parameters, including projection parameters, are trained.
All runs use Adam~\cite{kingma2014adam} with default hyperparameters for 20 epochs.

\paragraph{Two-stage training.}
This setup uses the same models but trains in two stages.
Stage 1 runs for 7 epochs with Adam using default hyperparameters.
Stage 2 runs for 13 epochs with SGD, learning rate $10^{-4}$, momentum 0.9, following the Keras Applications fine-tuning example \cite{chollet2015keras}.
Under this setup, we evaluate three combinations:
Logistic regression in Stage 1, then full fine-tuning in Stage 2 (LR+FT);
model  projection in Stage 1, then full fine-tuning in Stage 2 (Model projection + FT); and
model projection in both stages (2 step model Projection).

We perform the second setup because we observed full fine-tuning occasionally failed to learn in single-stage training, particularly for VGG16.
Two-stage training provides a stronger full fine-tuning baseline and allows us to study projection
as an alternative warm start to logistic regression training and/or an initialization before full fine-tuning.

\begin{figure*}[t]
    
\centering

\begin{tabular}{lccc}
& \textbf{CIFAR 10} & \textbf{CIFAR 100} & \textbf{Oxford Flowers} \\ 
\textbf{\rotatebox{90}{VGG16}} & \includegraphics[width=0.25\textwidth]{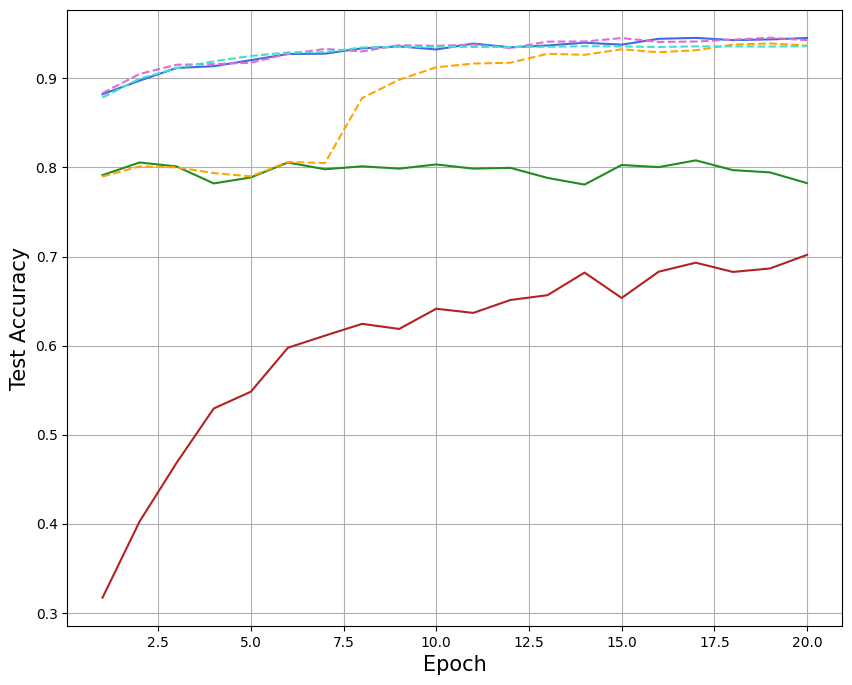} & \includegraphics[width=0.25\textwidth]{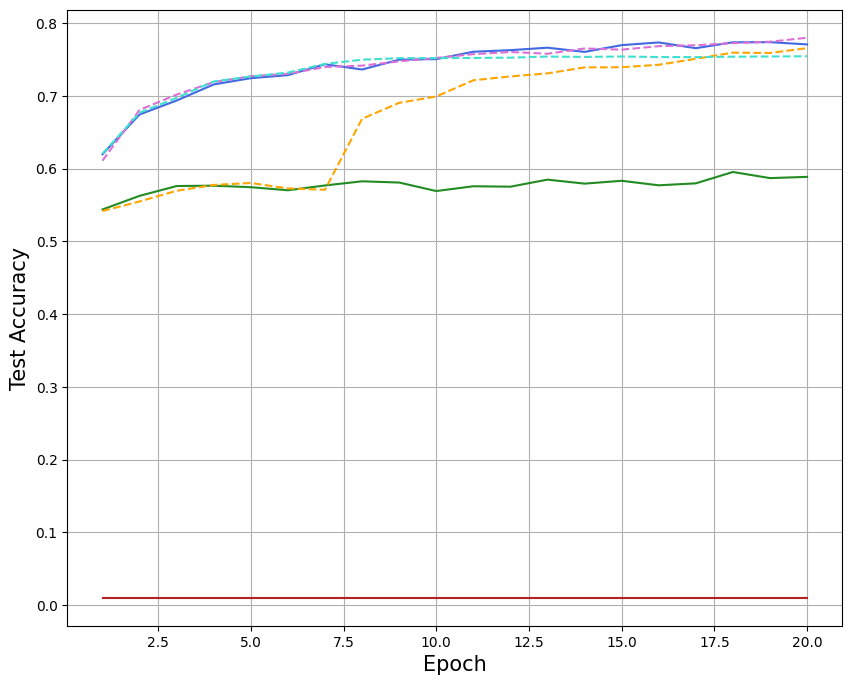} & \includegraphics[width=0.25\textwidth]{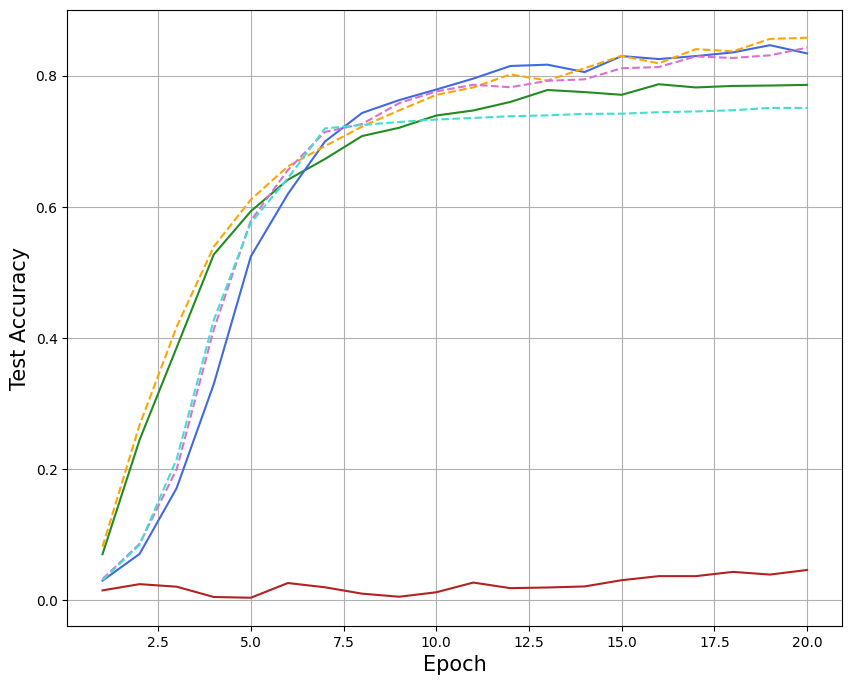} \\ 
\textbf{\rotatebox{90}{ResNet50}} & \includegraphics[width=0.25\textwidth]{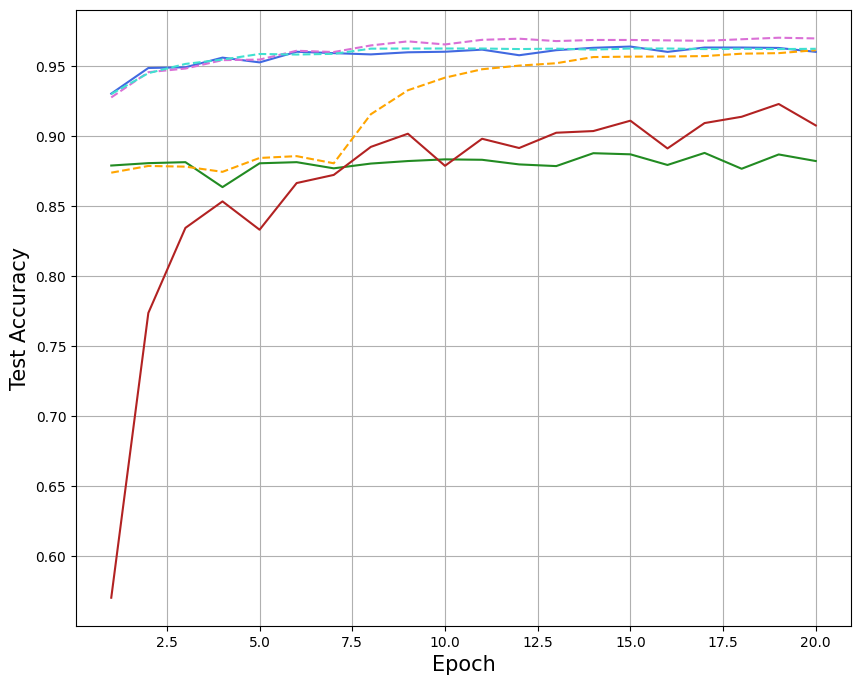} & \includegraphics[width=0.25\textwidth]{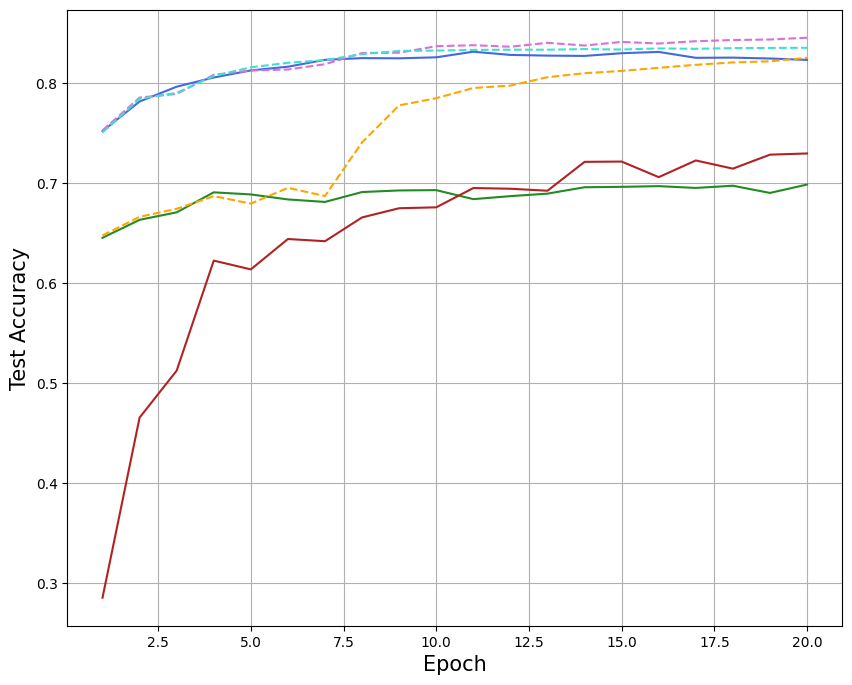} & \includegraphics[width=0.25\textwidth]{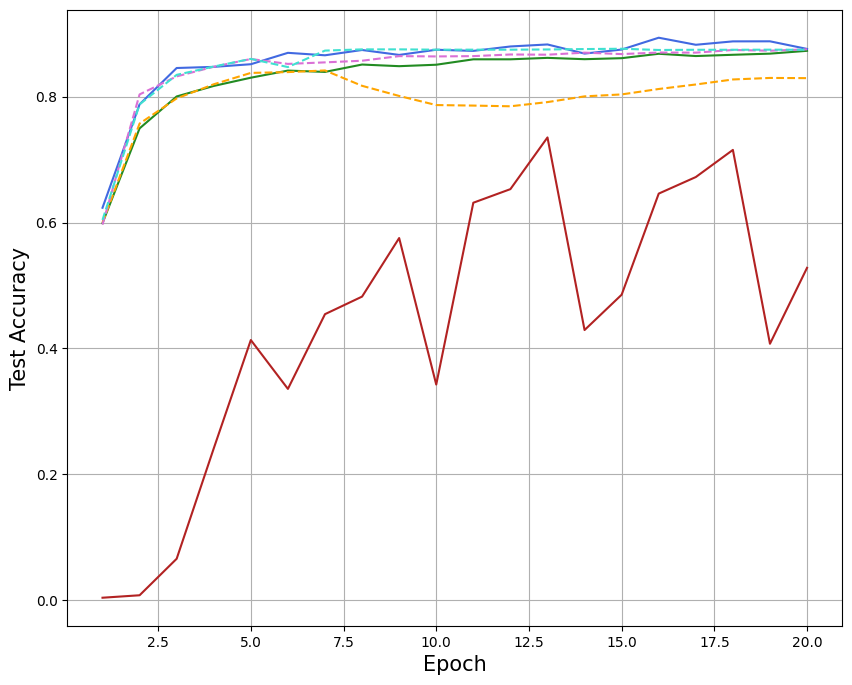} \\ 
\textbf{\rotatebox{90}{DenseNet121}} & \includegraphics[width=0.25\textwidth]{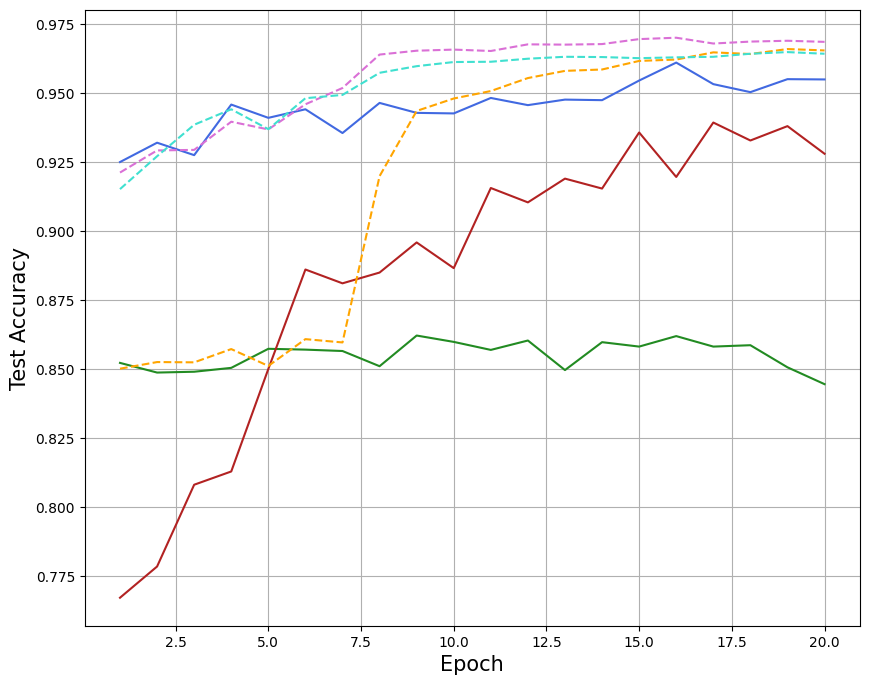} & \includegraphics[width=0.25\textwidth]{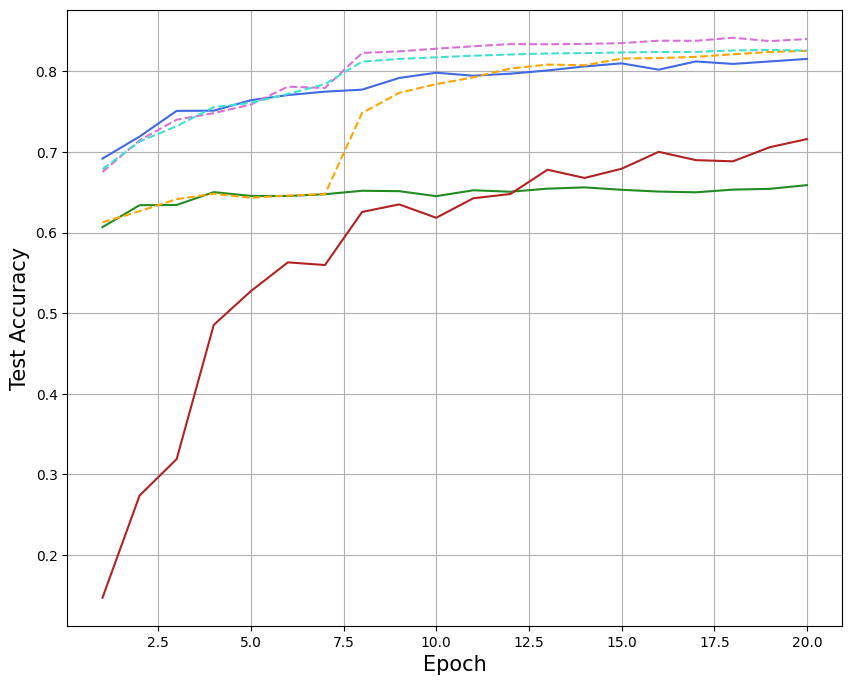} & \includegraphics[width=0.25\textwidth]{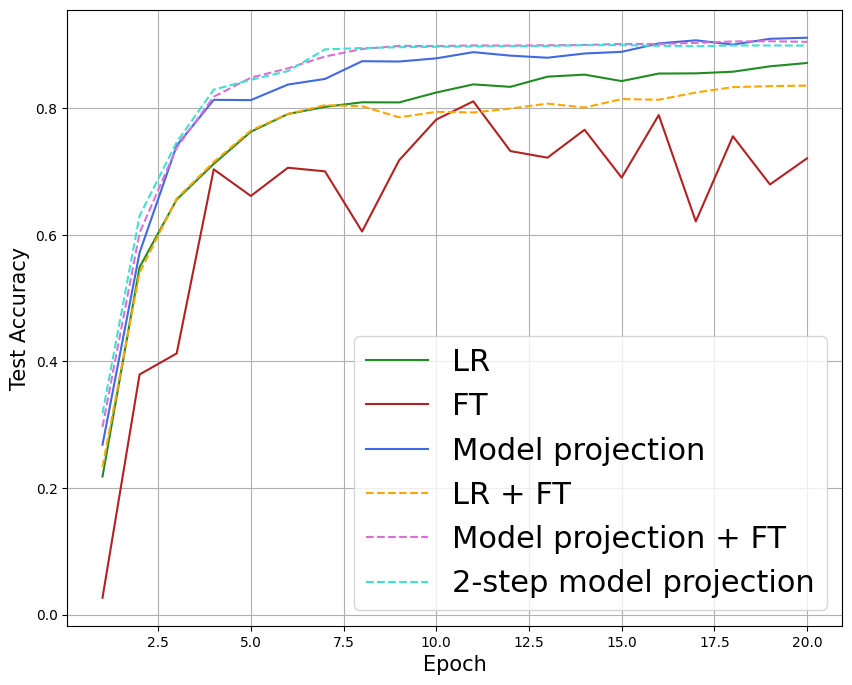} \\ 
\end{tabular}
\caption{Select results from the first experiments. 
The solid lines show the single stage setup, while the dashed lines show the two stage setup.
Each row corresponds to a particular convolutional base, organized from oldest to most recent. The columns show the results on CIFAR 10, CIFAR 100, and Oxford flowers respectively. In all charts, the x-axis is epochs, and the y-axis is test accuracy.}\label{fig:mainresults}
\end{figure*}

\begin{table}[!t]
\caption{Average accuracy and (standard deviation) of our method on the CIFAR 10, CIFAR 100, and Flowers datasets, compared to results from other recent methods, all using ResNet50. Values taken from other papers (LoRA-C \cite{ding2024lora} and PaRA, LoRA-r8, LoRA-r32 \cite{HEDEGAARD2024110724}) are marked with *. The best averages are bolded, along with any averages that are within two standard deviations of the best result.}
\label{tab:peft_baselines}
\centering
\resizebox{\textwidth}{!}{
\begin{tabular}{lccccc}
\hline
Method & \textbf{C10}& \textbf{C100} & \textbf{Flowers} \\ \hline
\textbf{LR}  
&  87.90 (0.621)  
&   69.73 (0.137)   
& \textbf{87.22} (0.161) \\ 

\textbf{Batch Normalization}  
&  95.41 (0.152)  
&   81.21 (0.099)   
& \textbf{87.81} (0.275) \\ 

\textbf{BitFit}  
&  95.49 (0.168)  
&  81.28 (0.173)    
& \textbf{88.31} (0.835) \\ 

\textbf{FT}  
&   91.67 (0.692) 
&  72.36 (0.807)    
& 63.60 (8.961) \\ 

\textbf{LR + FT}  
&  96.18 (0.166)  
&   82.57 (0.248)   
& 83.26 (0.253) \\ 
\hline

\textbf{Model projection}  
& 96.13 (0.118) 
&  82.58 (0.180)   
& \textbf{87.98} (1.049) \\ 

\textbf{Model projection + FT}  
& \textbf{96.83} (0.111) 
& \textbf{84.06} (0.335) 
& \textbf{87.69} (0.306)\\ 

\textbf{2-step model projection}  
&  96.35 (0.193)  
&   \textbf{83.53} (0.054)   
&  \textbf{87.17} (0.386)\\
\hline

\textbf{LoRA-C}   
& 96.59* (N/A*)  
& 82.98* (N/A*)  
& N/A* (N/A*) \\

\textbf{LoRA-r32}   
& 95.32* (0.13*)  
& 79.47* (N/A*)  
&  78.57* (5.93*)\\

\textbf{LoRA-r8} 
& 95.35* (0.10*)  
& 79.22* (N/A*)  
&  80.96* (5.83*)\\

\textbf{PaRA}   
& 93.52* (N/A*)  
& 79.31* (N/A*)  
&  80.65* (N/A*)\\
\hline
\end{tabular}
}
\end{table}

\textbf{Results. } Results are summarized in Figure~\ref{fig:mainresults}.
Overall, model projection consistently outperforms logistic regression and full fine-tuning across both large and small datasets, and remains robust under older architectures and a simple training setup, while still performing well on newer models.
Logistic regression 
typically achieves higher initial accuracy but shows limited improvement over epochs, indicating underfitting. Its performance is competitive on smaller datasets but degrades on larger ones, and it remains relatively robust on older models.
Full fine-tuning starts from lower accuracy and improves over training, sometimes surpassing logistic regression 
but often struggling on smaller datasets or failing to learn for older architectures such as VGG16.
In contrast, model projection combines favorable properties of both approaches: it starts from relatively high accuracy and continues to improve throughout training, achieving the strongest overall performance across experiments.

These results align with what we expect from projection inheritance: strong results from learning an entire neural network, allowing frequent outperformance of 
logistic regression and greater robustness to dataset size compared to full fine tuning due to inheriting fewer trainable parameters.
Model projection retains the expressiveness of end-to-end learning while improving robustness to dataset size and optimization instability.
While we expect full fine-tuning to outperform model projection given sufficiently large datasets, and logistic regression training to dominate in extremely low-data regimes~\cite{kornblith2019better}, our results suggest that model projection performs best across a broad intermediate regime that covers most practical transfer learning scenarios.

In the two-stage setting, the first stage behaves similarly to the corresponding single-stage experiments, as expected.
Notably, performance in the second stage is consistently higher when model projection is used in the first stage, regardless of whether the second stage applies model projection or full fine-tuning.
This suggests that model projection provides a more effective initialization for subsequent fine-tuning.
Overall, model projection yields stronger performance 
and emerges as the most effective strategy in the majority of cases.
Further details and full results can be found in Appendix \ref{supResults}.

\subsection{Second Experiment: Model projection Against PEFT Baselines}
To evaluate model projection against current 
PEFT approaches, we compare to standard transfer baselines and recent PEFT methods on CIFAR-10, CIFAR-100, and Oxford Flowers on pretrained ResNet50 architecture. Our baselines include logistic regression (training only the final classifier) and full fine-tuning, as well as partial-tuning baselines that update only BatchNorm parameters (BN) \citep{ioffe2015batch} along with the final layer, or only bias terms (inspired by BitFit) \citep{zaken2022bitfit} along with BN and the final layer. This gives a set of methods with gradually increasing trainable parameters on the same model. We additionally compare to low-rank adaptation methods LoRA-r8 and LoRA-r32 \citep{hu2022lora}, the Parallel Residual Adapter (PaRA) family \citep{rebuffi2017learning}, and the convolutional variant LoRA-C \citep{ding2024lora}. For LoRA-C we report results from \citet{ding2024lora}, and for LoRA-r8/LoRA-r32 and PaRA we use the ResNet50 results reported by \citet{HEDEGAARD2024110724}. We include all two-stage variants considered previously. We report the average accuracy across 5 seeds, along with the standard deviation. 

\textbf{Results. }Table~\ref{tab:peft_baselines} summarizes the results.
Both single-stage and two-stage model projection are highly competitive across all datasets when compared to PEFT baselines.
In particular, the two-stage setup that uses model projection in the first stage, either with full fine tuning or model projection in the second stage, achieves the best performance on CIFAR-10 and CIFAR-100, while all projected models are within two standard deviations of the best performance on Flowers.
These results support the view that model projection provides an effective initialization, especially for subsequent full fine-tuning, by guiding optimization toward more favorable regions of the loss landscape.

\subsection*{Conclusions.} In this paper we formalized a general version of the operations of nodes from traditional neural networks and convolutional neural networks. We used this formalization to prove that FFNs can inherit from CNNs, but that CNNs cannot inherit from FFNs.
We proposed \emph{model projection}, a method to project a CNN node onto an FFN node, and proved that this method allows a CNN node to inherit any results or techniques from FFN nodes that do not require homogeneous inputs.
We then demonstrated a proof of concept of our proposed projection method in a transfer learning example. Our method had strong results in a simple setting, reinforcing our theoretical findings. Our example is a good application of projection that we expect will be useful for simple home made models, and may also be useful in more heavily engineered settings as well.
There are a number of potential directions for future research. Extending analysis of inheritance between more model types such as graph neural networks and transformer models is particularly interesting. More results from FFNs can be inherited to CNNs, and exploration of some of these methods such as lasso artificial neural networks \cite{ma2022phase} would be interesting to experiment with. Additionally, further exploration of our example experiment could be done in less simple settings.



\bibliography{references.bib}

\begin{thebibliography}{36}
\providecommand{\natexlab}[1]{#1}
\providecommand{\url}[1]{\texttt{#1}}
\expandafter\ifx\csname urlstyle\endcsname\relax
  \providecommand{\doi}[1]{doi: #1}\else
  \providecommand{\doi}{doi: \begingroup \urlstyle{rm}\Url}\fi

\bibitem[Abadi et~al.(2015)Abadi, Agarwal, Barham, Brevdo, Chen, Citro, Corrado, Davis, Dean, Devin, Ghemawat, Goodfellow, Harp, Irving, Isard, Jia, Jozefowicz, Kaiser, Kudlur, Levenberg, Man\'{e}, Monga, Moore, Murray, Olah, Schuster, Shlens, Steiner, Sutskever, Talwar, Tucker, Vanhoucke, Vasudevan, Vi\'{e}gas, Vinyals, Warden, Wattenberg, Wicke, Yu, and Zheng]{tensorflow2015-whitepaper}
M.~Abadi, A.~Agarwal, P.~Barham, E.~Brevdo, Z.~Chen, C.~Citro, G.~S. Corrado, A.~Davis, J.~Dean, M.~Devin, S.~Ghemawat, I.~Goodfellow, A.~Harp, G.~Irving, M.~Isard, Y.~Jia, R.~Jozefowicz, L.~Kaiser, M.~Kudlur, J.~Levenberg, D.~Man\'{e}, R.~Monga, S.~Moore, D.~Murray, C.~Olah, M.~Schuster, J.~Shlens, B.~Steiner, I.~Sutskever, K.~Talwar, P.~Tucker, V.~Vanhoucke, V.~Vasudevan, F.~Vi\'{e}gas, O.~Vinyals, P.~Warden, M.~Wattenberg, M.~Wicke, Y.~Yu, and X.~Zheng.
\newblock {TensorFlow}: Large-scale machine learning on heterogeneous systems, 2015.
\newblock URL \url{https://www.tensorflow.org/}.
\newblock Software available from tensorflow.org.

\bibitem[Ansel et~al.(2024)Ansel, Yang, He, Gimelshein, Jain, Voznesensky, Bao, Bell, Berard, Burovski, Chauhan, Chourdia, Constable, Desmaison, DeVito, Ellison, Feng, Gong, Gschwind, Hirsh, Huang, Kalambarkar, Kirsch, Lazos, Lezcano, Liang, Liang, Lu, Luk, Maher, Pan, Puhrsch, Reso, Saroufim, Siraichi, Suk, Suo, Tillet, Wang, Wang, Wen, Zhang, Zhao, Zhou, Zou, Mathews, Chanan, Wu, and Chintala]{Ansel_PyTorch_2_Faster_2024}
J.~Ansel, E.~Yang, H.~He, N.~Gimelshein, A.~Jain, M.~Voznesensky, B.~Bao, P.~Bell, D.~Berard, E.~Burovski, G.~Chauhan, A.~Chourdia, W.~Constable, A.~Desmaison, Z.~DeVito, E.~Ellison, W.~Feng, J.~Gong, M.~Gschwind, B.~Hirsh, S.~Huang, K.~Kalambarkar, L.~Kirsch, M.~Lazos, M.~Lezcano, Y.~Liang, J.~Liang, Y.~Lu, C.~Luk, B.~Maher, Y.~Pan, C.~Puhrsch, M.~Reso, M.~Saroufim, M.~Y. Siraichi, H.~Suk, M.~Suo, P.~Tillet, E.~Wang, X.~Wang, W.~Wen, S.~Zhang, X.~Zhao, K.~Zhou, R.~Zou, A.~Mathews, G.~Chanan, P.~Wu, and S.~Chintala.
\newblock {PyTorch 2: Faster Machine Learning Through Dynamic Python Bytecode Transformation and Graph Compilation}.
\newblock In \emph{29th ACM International Conference on Architectural Support for Programming Languages and Operating Systems, Volume 2 (ASPLOS '24)}. ACM, Apr. 2024.
\newblock \doi{10.1145/3620665.3640366}.
\newblock URL \url{https://docs.pytorch.org/assets/pytorch2-2.pdf}.

\bibitem[Bossard et~al.(2014)Bossard, Guillaumin, and Van~Gool]{bossard14}
L.~Bossard, M.~Guillaumin, and L.~Van~Gool.
\newblock Food-101 -- mining discriminative components with random forests.
\newblock In \emph{European Conference on Computer Vision}, 2014.

\bibitem[Chen et~al.(2024)Chen, Tao, Zhang, Wang, Li, Ye, Wang, Hu, and Savvides]{chen2024conv}
H.~Chen, R.~Tao, H.~Zhang, Y.~Wang, X.~Li, W.~Ye, J.~Wang, G.~Hu, and M.~Savvides.
\newblock Conv-adapter: Exploring parameter efficient transfer learning for convnets.
\newblock In \emph{Proceedings of the IEEE/CVF conference on computer vision and pattern recognition}, pages 1551--1561, 2024.

\bibitem[Chollet(2021)]{chollet2021deep}
F.~Chollet.
\newblock \emph{Deep Learning with Python}.
\newblock Manning, 2021.

\bibitem[Chollet et~al.(2015)]{chollet2015keras}
F.~Chollet et~al.
\newblock Keras.
\newblock \url{https://keras.io}, 2015.

\bibitem[Deng et~al.(2009)Deng, Dong, Socher, Li, Li, and Fei-Fei]{deng2009imagenet}
J.~Deng, W.~Dong, R.~Socher, L.-J. Li, K.~Li, and L.~Fei-Fei.
\newblock Imagenet: A large-scale hierarchical image database.
\newblock In \emph{2009 IEEE conference on computer vision and pattern recognition}, pages 248--255. Ieee, 2009.

\bibitem[Ding et~al.(2024)Ding, Cao, Xie, Fan, Wang, and Lu]{ding2024lora}
C.~Ding, X.~Cao, J.~Xie, L.~Fan, S.~Wang, and Z.~Lu.
\newblock Lora-c: Parameter-efficient fine-tuning of robust cnn for iot devices.
\newblock \emph{arXiv preprint arXiv:2410.16954}, 2024.

\bibitem[Fei-Fei et~al.(2004)Fei-Fei, Fergus, and Perona]{FeiFei2004LearningGV}
L.~Fei-Fei, R.~Fergus, and P.~Perona.
\newblock Learning generative visual models from few training examples: An incremental bayesian approach tested on 101 object categories.
\newblock \emph{Computer Vision and Pattern Recognition Workshop}, 2004.

\bibitem[Goodfellow et~al.(2016)Goodfellow, Bengio, and Courville]{goodfellow2016deep}
I.~Goodfellow, Y.~Bengio, and A.~Courville.
\newblock \emph{Deep learning}.
\newblock MIT press, 2016.

\bibitem[Han et~al.(2024)Han, Wang, Cui, Wang, Huang, Qi, and Liu]{han2024facing}
C.~Han, Q.~Wang, Y.~Cui, W.~Wang, L.~Huang, S.~Qi, and D.~Liu.
\newblock Facing the elephant in the room: Visual prompt tuning or full finetuning?
\newblock In \emph{The Twelfth International Conference on Learning Representations}, 2024.

\bibitem[Hastie et~al.(2009)Hastie, Tibshirani, and Friedman]{hastieESL}
T.~Hastie, R.~Tibshirani, and J.~H. Friedman.
\newblock \emph{The elements of statistical learning: data mining, inference, and prediction}, volume~2.
\newblock Springer, 2009.

\bibitem[He et~al.(2023)He, Cai, Zhang, Tao, and Zhuang]{he2023sensitivity}
H.~He, J.~Cai, J.~Zhang, D.~Tao, and B.~Zhuang.
\newblock Sensitivity-aware visual parameter-efficient fine-tuning.
\newblock In \emph{Proceedings of the IEEE/CVF International Conference on Computer Vision}, pages 11825--11835, 2023.

\bibitem[He et~al.(2016)He, Zhang, Ren, and Sun]{he2016deep}
K.~He, X.~Zhang, S.~Ren, and J.~Sun.
\newblock Deep residual learning for image recognition.
\newblock In \emph{Proceedings of the IEEE conference on computer vision and pattern recognition}, pages 770--778, 2016.

\bibitem[Hedegaard et~al.(2024)Hedegaard, Alok, Jose, and Iosifidis]{HEDEGAARD2024110724}
L.~Hedegaard, A.~Alok, J.~Jose, and A.~Iosifidis.
\newblock Structured pruning adapters.
\newblock \emph{Pattern Recognition}, 156:\penalty0 110724, 2024.
\newblock ISSN 0031-3203.
\newblock \doi{https://doi.org/10.1016/j.patcog.2024.110724}.
\newblock URL \url{https://www.sciencedirect.com/science/article/pii/S0031320324004758}.

\bibitem[Howard et~al.(2017)Howard, Zhu, Chen, Kalenichenko, Wang, Weyand, Andreetto, and Adam]{howard2017mobilenets}
A.~G. Howard, M.~Zhu, B.~Chen, D.~Kalenichenko, W.~Wang, T.~Weyand, M.~Andreetto, and H.~Adam.
\newblock Mobilenets: Efficient convolutional neural networks for mobile vision applications.
\newblock \emph{arXiv preprint arXiv:1704.04861}, 2017.

\bibitem[Hu et~al.(2022)Hu, Shen, Wallis, Allen-Zhu, Li, Wang, Wang, Chen, et~al.]{hu2022lora}
E.~J. Hu, Y.~Shen, P.~Wallis, Z.~Allen-Zhu, Y.~Li, S.~Wang, L.~Wang, W.~Chen, et~al.
\newblock Lora: Low-rank adaptation of large language models.
\newblock \emph{ICLR}, 1\penalty0 (2):\penalty0 3, 2022.

\bibitem[Huang et~al.(2017)Huang, Liu, Van Der~Maaten, and Weinberger]{huang2017densely}
G.~Huang, Z.~Liu, L.~Van Der~Maaten, and K.~Q. Weinberger.
\newblock Densely connected convolutional networks.
\newblock In \emph{Proceedings of the IEEE conference on computer vision and pattern recognition}, pages 4700--4708, 2017.

\bibitem[Ioffe and Szegedy(2015)]{ioffe2015batch}
S.~Ioffe and C.~Szegedy.
\newblock Batch normalization: Accelerating deep network training by reducing internal covariate shift.
\newblock In \emph{International conference on machine learning}, pages 448--456. pmlr, 2015.

\bibitem[Jia et~al.(2022)Jia, Tang, Chen, Cardie, Belongie, Hariharan, and Lim]{jia2022visual}
M.~Jia, L.~Tang, B.-C. Chen, C.~Cardie, S.~Belongie, B.~Hariharan, and S.-N. Lim.
\newblock Visual prompt tuning.
\newblock In \emph{European conference on computer vision}, pages 709--727. Springer, 2022.

\bibitem[Khosla et~al.(2011)Khosla, Jayadevaprakash, Yao, and Fei-Fei]{stanford_dogs_FGVC2011}
A.~Khosla, N.~Jayadevaprakash, B.~Yao, and L.~Fei-Fei.
\newblock Novel dataset for fine-grained image categorization.
\newblock In \emph{First Workshop on Fine-Grained Visual Categorization, IEEE Conference on Computer Vision and Pattern Recognition}, Colorado Springs, CO, June 2011.

\bibitem[Kingma(2015)]{kingma2014adam}
D.~P. Kingma.
\newblock Adam: A method for stochastic optimization.
\newblock In \emph{International Conference on Learning Representations (ICLR 2015)}, 2015.

\bibitem[Kornblith et~al.(2019)Kornblith, Shlens, and Le]{kornblith2019better}
S.~Kornblith, J.~Shlens, and Q.~V. Le.
\newblock Do better imagenet models transfer better?
\newblock In \emph{Proceedings of the IEEE/CVF conference on computer vision and pattern recognition}, pages 2661--2671, 2019.

\bibitem[Krizhevsky(2009)]{Krizhevsky09learningmultiple}
A.~Krizhevsky.
\newblock Learning multiple layers of features from tiny images.
\newblock Technical report, University of Toronto, 2009.

\bibitem[LeCun et~al.(2010)LeCun, Kavukcuoglu, and Farabet]{lecun2010convolutional}
Y.~LeCun, K.~Kavukcuoglu, and C.~Farabet.
\newblock Convolutional networks and applications in vision.
\newblock In \emph{Proceedings of 2010 IEEE international symposium on circuits and systems}, pages 253--256. IEEE, 2010.

\bibitem[Lian et~al.(2022)Lian, Zhou, Feng, and Wang]{lian2022scaling}
D.~Lian, D.~Zhou, J.~Feng, and X.~Wang.
\newblock Scaling \& shifting your features: A new baseline for efficient model tuning.
\newblock \emph{Advances in Neural Information Processing Systems}, 35:\penalty0 109--123, 2022.

\bibitem[Linse et~al.(2023)Linse, Barth, and Martinetz]{linse2023convolutional}
C.~Linse, E.~Barth, and T.~Martinetz.
\newblock Convolutional neural networks do work with pre-defined filters.
\newblock In \emph{2023 International Joint Conference on Neural Networks (IJCNN)}, pages 1--8. IEEE, 2023.

\bibitem[Linse et~al.(2024)Linse, Br{\"u}ckner, and Martinetz]{linse2024enhancing}
C.~Linse, B.~Br{\"u}ckner, and T.~Martinetz.
\newblock Enhancing generalization in convolutional neural networks through regularization with edge and line features.
\newblock In \emph{International Conference on Artificial Neural Networks}, pages 432--446. Springer, 2024.

\bibitem[Luo et~al.(2023)Luo, Huang, Zhou, Sun, Jiang, Wang, and Ji]{luo2023towards}
G.~Luo, M.~Huang, Y.~Zhou, X.~Sun, G.~Jiang, Z.~Wang, and R.~Ji.
\newblock Towards efficient visual adaption via structural re-parameterization.
\newblock \emph{arXiv preprint arXiv:2302.08106}, 2023.

\bibitem[Ma et~al.(2022)Ma, Sardy, Hengartner, Bobenko, and Lin]{ma2022phase}
X.~Ma, S.~Sardy, N.~Hengartner, N.~Bobenko, and Y.~T. Lin.
\newblock A phase transition for finding needles in nonlinear haystacks with lasso artificial neural networks.
\newblock \emph{Statistics and Computing}, 32\penalty0 (6):\penalty0 99, 2022.

\bibitem[Nilsback and Zisserman(2008)]{Nilsback08}
M.-E. Nilsback and A.~Zisserman.
\newblock Automated flower classification over a large number of classes.
\newblock In \emph{Proceedings of the Indian Conference on Computer Vision, Graphics and Image Processing}, Dec 2008.

\bibitem[Parkhi et~al.(2012)Parkhi, Vedaldi, Zisserman, and Jawahar]{parkhi12a}
O.~M. Parkhi, A.~Vedaldi, A.~Zisserman, and C.~V. Jawahar.
\newblock Cats and dogs.
\newblock In \emph{IEEE Conference on Computer Vision and Pattern Recognition}, 2012.

\bibitem[Rebuffi et~al.(2017)Rebuffi, Bilen, and Vedaldi]{rebuffi2017learning}
S.-A. Rebuffi, H.~Bilen, and A.~Vedaldi.
\newblock Learning multiple visual domains with residual adapters.
\newblock \emph{Advances in neural information processing systems}, 30, 2017.

\bibitem[Simonyan and Zisserman(2015)]{simonyan2015very}
K.~Simonyan and A.~Zisserman.
\newblock Very deep convolutional networks for large-scale image recognition.
\newblock In \emph{3rd International Conference on Learning Representations (ICLR 2015)}. Computational and Biological Learning Society, 2015.

\bibitem[Wang et~al.(2023)Wang, Wang, Wang, Li, Da, Liu, Gao, Shen, He, Shen, et~al.]{wang2023real}
D.~Wang, X.~Wang, L.~Wang, M.~Li, Q.~Da, X.~Liu, X.~Gao, J.~Shen, J.~He, T.~Shen, et~al.
\newblock A real-world dataset and benchmark for foundation model adaptation in medical image classification.
\newblock \emph{Scientific Data}, 10\penalty0 (1):\penalty0 574, 2023.

\bibitem[Zaken et~al.(2022)Zaken, Goldberg, and Ravfogel]{zaken2022bitfit}
E.~B. Zaken, Y.~Goldberg, and S.~Ravfogel.
\newblock Bitfit: Simple parameter-efficient fine-tuning for transformer-based masked language-models.
\newblock In \emph{Proceedings of the 60th Annual Meeting of the Association for Computational Linguistics (Volume 2: Short Papers)}, pages 1--9, 2022.

\end{thebibliography}

 \newpage
\appendix
\onecolumn
\section{Supplementary proofs}\label{supProofs}
\begin{proof}[Proof of Theorem~\ref{subsetsMain}]
We will prove the theorem by showing a bijection between GFFNs and a subset of GCNNs. 
Suppose that we have a GCNN node where $F_{jk}$ has size $1$ in all dimensions. 
This means $F_{jk}$ is a scalar, and $F_{j}$ is a vector of scalars. As a result, $Z_{k} \ast F_{jk}=Z_{k} F_{jk}$. 
Using this fact, and the definition of GCNN nodes, we can write:
\begin{align*}
f_{j}(Z, W_{j}) = \sigma(Z \ast F_{j} + b_{j}\Gamma_{D_{out}}) = \sigma(\sum_{k=1}^{d} Z_{k} \ast F_{jk} + b_{j}\Gamma_{D_{out}}) = \sigma(\sum_{k=1}^{d} Z_{k} F_{jk} + b_{j}\Gamma_{D_{out}}).
\end{align*}
Note that $\sum_{k=1}^{d} Z_{k} F_{jk}=Z \odot_t F_{j}$ by definition. 
As a result, we have that 
$$\sigma(\sum_{k=1}^{d} Z_{k} F_{jk} + b_{j}\Gamma_{D_{out}})=\sigma(Z \odot_t F_{j} + b_{j}\Gamma_{D_{out}}).$$
Observe that the right-hand side is exactly the definition of GFFN node. Every GCNN node with $F_{jk}$ of size $1$ in all dimensions maps to exactly one GFFN node. We now show the inverse. Suppose we have a GFFN node, viz.
\begin{align*}
    f_{j}(Z, W_{j}) &= \sigma(Z \odot_t W_{j} + b_{j}\Gamma_{D_{out}})= \sigma(\sum_{k=1}^{d} Z_{k}  W_{jk} + b_{j}\Gamma_{D_{out}}).
\end{align*}
Note that since $W_{jk}$ is a scalar, $\sum_{k=1}^{d} Z_{k}  W_{jk} = \sum_{k=1}^{d} Z_{k}\ast W_{jk}$. As a result, we have that
\begin{align*}
    \sigma(\sum_{k=1}^{d} Z_{k}  W_{jk} + b_{j}\Gamma_{D_{out}}) = \sigma(\sum_{k=1}^{d} Z_{k} \ast W_{jk} + b_{j}\Gamma_{D_{out}}).
\end{align*}
Observe that the right hand side is exactly the definition of a GCNN node, where the kernels have size $1$ in all dimensions. Every GFFN node maps to exactly one GCNN node with kernels of size $1$.

Therefore, there is a bijection between the subset of GCNNs with kernel size $1$ in all dimensions, and the set of GFFNs. Therefore, GFFNs are a subset of GCNNs.
\end{proof}

\subsection{Projection}

\begin{lemma}\label{separablCNNs}
GCNN node functions are separable by input.
\end{lemma}
\begin{proof}
From Definition \ref{gCNN} we have that GCNN node functions can be written as: 
\begin{equation*}
    \sum_{k=1}^{d} Z_{k} \ast F_{jk} + b_{j}\Gamma_{D_{out}} = \sum_{k=1}^{d} f_{jk}'(Z_{k}, W_{jk}) + b_{j}\Gamma_{D_{out}}.
\end{equation*}
 Here, $f_{jk}'$ is a convolution, and $W_{jk}$ are the weights in filter channel $F_{jk}$. 
\end{proof}

\begin{lemma}\label{weightPlacement}
If all node functions are separable by input, and all node sub-functions, $f_{jk}'$, are associative for scalar multiplication, then $\gamma_{jk}$ can be applied either before, or after $f_{jk}'$ acts on $Z_{k}$.
\end{lemma}
\begin{proof}
 Each node is separable by input, therefore, we can just consider each node sub-function. Since all node sub-functions, $f_{jk}'$, are associative for scalar multiplication, then $f_{jk}'(\gamma_{jk} Z_{k}) = \gamma_{jk} f_{jk}'(Z_{k})$.
\end{proof}

\begin{proposition}\label{gCNNweightPlacement}
In a GCNN node, $\gamma_{jk}$ can be applied either before, or after $f_{jk}'$ acts on $Z_{k}$.
\end{proposition}
\begin{proof}
Convolutions are associative for scalar multiplication, therefore, this follows from Lemmas \ref{separablCNNs} and \ref{weightPlacement}.
\end{proof}

\section{Supplementary remarks}\label{supRemarks}
\begin{remark}\label{pre-definedCNN}
    A projected CNN differs from a pre-defined CNN \cite{linse2023convolutional}, where a set of fixed filters produce outputs that are then sent to a layer of 1 x 1 convolutional nodes. Instead, each projected node receives a unique input from a single fixed set of filter channels that have not aggregated their channels.
    Specifically, a Pre-defined Filter Module is similar to a depthwise separable convolution, except they fix the filters in the depthwise part \cite{linse2024enhancing}.
    A depthwise separable convolution has a depthwise convolution layer (a single filter per input channel) followed by BN and RELU, then a 1x1 convolution layer, and another BN and RELU \cite{howard2017mobilenets}. 
    So the PDFM consists of 6 steps: a fixed depthwise convolution, BN, RELU, 1x1 convolution, BN, RELU. In contrast, our method has 2 steps: projected convolution, RELU. There is no BN, and no intermediary RELU. Also, beyond this, PDFMs only use pre-defined edge and line filters, which we do not.
    Furthermore, each 1x1 convolutional node of the depthwise separable receives the same input, whereas each projected node receives a unique input. A depthwise layer has one filter per input channel, where we have one filter for each input channel - output channel pair.
\end{remark}

\section{Supplementary results}\label{supResults}

\begin{table}[t]
    \centering
    \caption{Benchmark datasets.}
    \label{tab:datasets_lines}
    \begin{tabular}{ l c c c }
        \hline 
        \textbf{Dataset} & \textbf{Classes} & \textbf{Train} & \textbf{Test} \\
        \hline 
        Food 101 \nocite{bossard14} & 101 & 75750 & 25250 \\
        CIFAR 10 \nocite{Krizhevsky09learningmultiple}  & 10 & 50000 & 10000 \\
        CIFAR 100 \nocite{Krizhevsky09learningmultiple}  & 100 & 50000 & 10000 \\
        Stanford dogs \nocite{stanford_dogs_FGVC2011}  & 120 & 12000 & 8580 \\
        Oxford IIIT pets \nocite{parkhi12a}  & 37 & 3680 & 3669 \\
        Caltech 101 \nocite{FeiFei2004LearningGV}  & 102 & 3060 & 6084 \\
        Oxford 102 flowers \nocite{Nilsback08}  & 102 & 2040 & 6149 \\
        \hline 
    \end{tabular}
\end{table}

We now discuss the results from the first set of single stage experiments, seen in Figure \ref{fig:1stExp}, with summarizing information presented in Tables \ref{exp1bymethod}, \ref{exp1Bymodel}, and \ref{exp1ByData}.
Figure \ref{fig:1stExp} shows that model projection consistently had strong results, while training was relatively smooth and fast.
Table \ref{exp1bymethod} shows the ranking of each method across each experiment. Model projection significantly outperformed both the logistic regression and full fine tuning in this simple setup, producing the best model about $81\%$ of the time, and the second best about $19\%$ of the time.

Table \ref{exp1Bymodel} shows the rankings achieved by model projection organized by model. Model projection had consistently strong results on each model.
Table \ref{exp1ByData} shows the rankings for model projection grouped by dataset. Model projection had the best scores on all five datasets that were not exclusively pet related, but did not do as well on the two datasets that were exclusively pet related. Interestingly, model projection did well on datasets both larger and smaller than the pet ones, suggesting the issue is not solely dataset size.





\begin{table}[h!]
\centering
\caption{Ranking of each training method in the 21 experiments from the first experimental setup. In a particular experiment (setup, convolutional base, and dataset combination), the method producing the highest test accuracy at the end of the last epoch was designated '1st', the second highest '2nd', and so on.}\label{exp1bymethod}
\begin{tabular}{lccc}
\hline
Method & \textbf{1st} & \textbf{2nd} & \textbf{3rd} \\ \hline
\textbf{Model projection} & 17 & 4 & 0 \\ 
\textbf{Logistic regression} & 4 & 11 & 6 \\ 
\textbf{Full fine tuning} & 0 & 6 & 15 \\ 
\hline
\end{tabular}

\end{table}

\begin{table}[h!]
\centering
\caption{Ranking of Model projection in the 21 experiments from the first experimental setup, organized by convolutional base. In a particular experiment (setup, convolutional base, and dataset combination), the method producing the highest test accuracy at the end of the last epoch was designated '1st', the second highest '2nd', and so on.}\label{exp1Bymodel}
\begin{tabular}{lccc}
\hline
Model & \textbf{1st} & \textbf{2nd} & \textbf{3rd} \\ \hline
\textbf{VGG16} & 6 & 1 & 0 \\ 
\textbf{ResNet50} & 6 & 1 & 0 \\ 
\textbf{DenseNet121} & 5 & 2 & 0 \\ 
\hline
\end{tabular}

\end{table}
\newpage

\begin{table}[h!]
\centering
\caption{Ranking of Model projection in the 21 experiments from the first experimental setup, organized by dataset. In a particular experiment (setup, convolutional base, and dataset combination), the method producing the highest test accuracy at the end of the last epoch was designated '1st', the second highest '2nd', and so on.}\label{exp1ByData}
\begin{tabular}{lccc}
\hline
Dataset & \textbf{1st} & \textbf{2nd} & \textbf{3rd} \\ \hline
\textbf{Food 101} & 3 & 0 & 0 \\ 
\textbf{CIFAR 10} & 3 & 0 & 0 \\ 
\textbf{CIFAR 100} & 3 & 0 & 0 \\ 
\textbf{Stanford dogs} & 1 & 2 & 0 \\ 
\textbf{Oxford IIIT pets} & 1 & 2 & 0 \\
\textbf{Caltech 101} & 3 & 0 & 0 \\ 
\textbf{Oxford 102 flowers} & 3 & 0 & 0 \\ 
\hline
\end{tabular}

\end{table}
\newpage
\begin{figure}[h!]
    \centering
    \caption{Results from the first set of experiments. The blue lines show the performance of model projection. The green and red lines show the performances of the single layer logistic regression, and the full fine tuning respectively. Each column corresponds to a particular convolutional base, organized from oldest to most recent. Each row shows the results on a particular dataset, organized from largest to smallest. In all charts, the x-axis is epochs, and the y-axis is test accuracy.}
    \label{fig:1stExp}
    \newlength{\imagewidth}
    \setlength{\imagewidth}{0.25\textwidth}

    \begin{tabular}{l c c c} 
        & \textbf{VGG16} & \textbf{ResNet50} & \textbf{DenseNet121} \\
        \rule{0pt}{15pt}\\
        
        \textbf{Food} & 
        \includegraphics[width=\imagewidth]{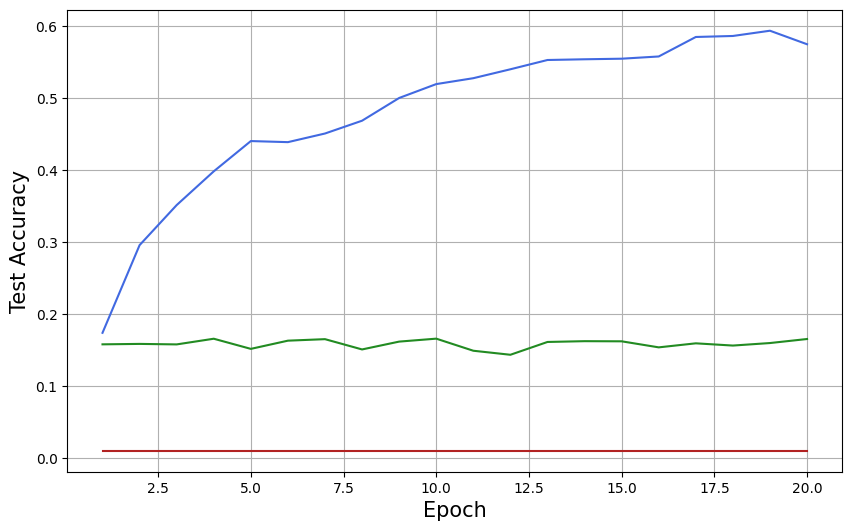} & 
        \includegraphics[width=\imagewidth]{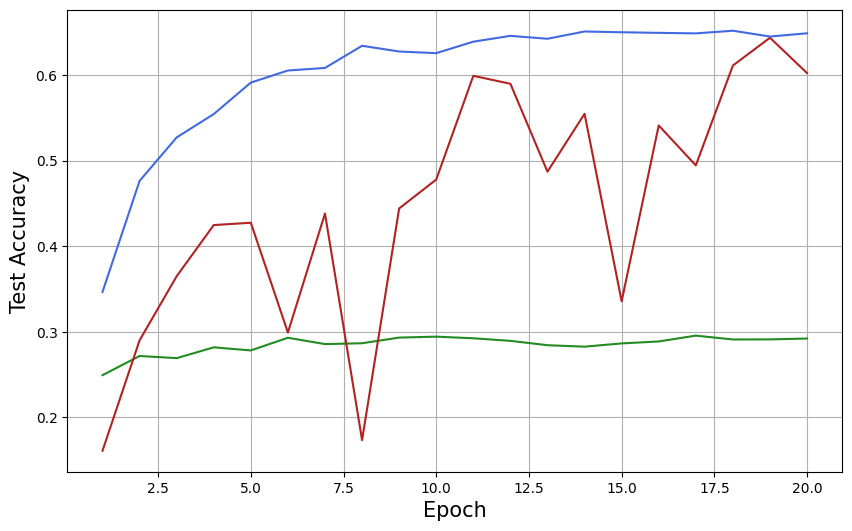} & 
        \includegraphics[width=\imagewidth]{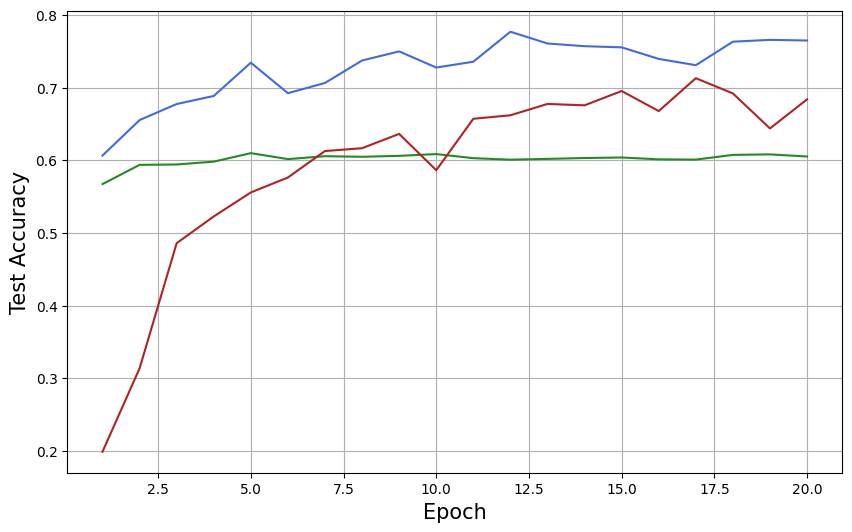} \\
        \rule{0pt}{10pt}\\
    
        \textbf{CIFAR 10} & 
        \includegraphics[width=\imagewidth]{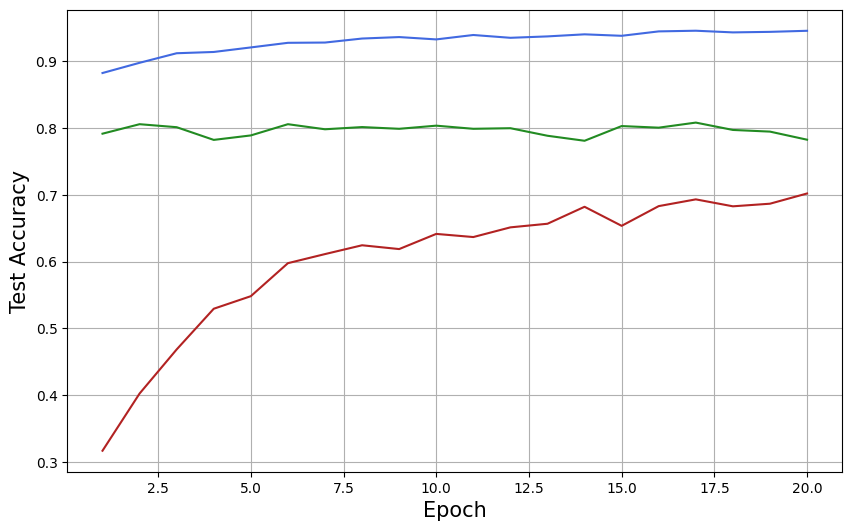} & 
        \includegraphics[width=\imagewidth]{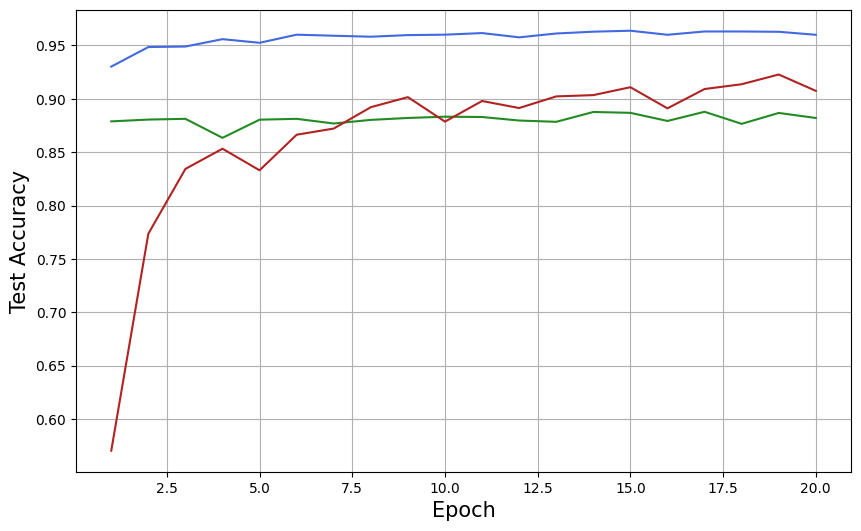} & 
        \includegraphics[width=\imagewidth]{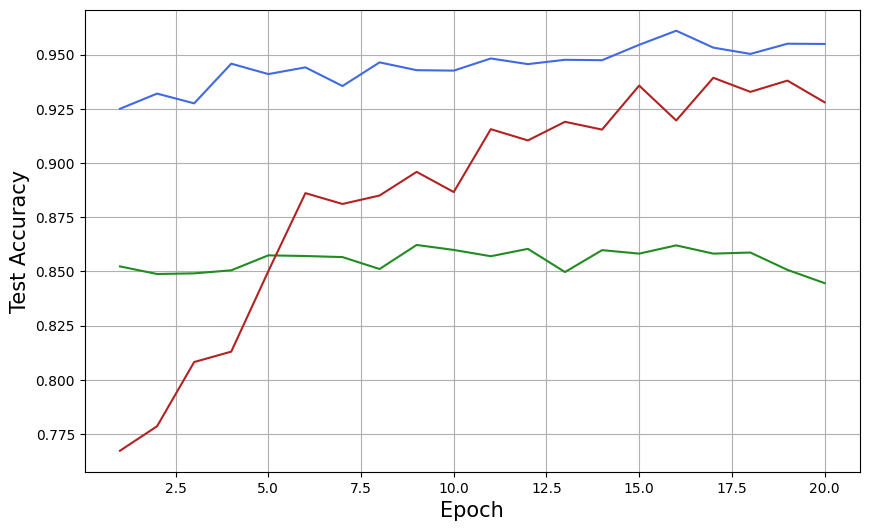} \\
        \rule{0pt}{10pt}\\
    
        \textbf{CIFAR 100} & 
        \includegraphics[width=\imagewidth]{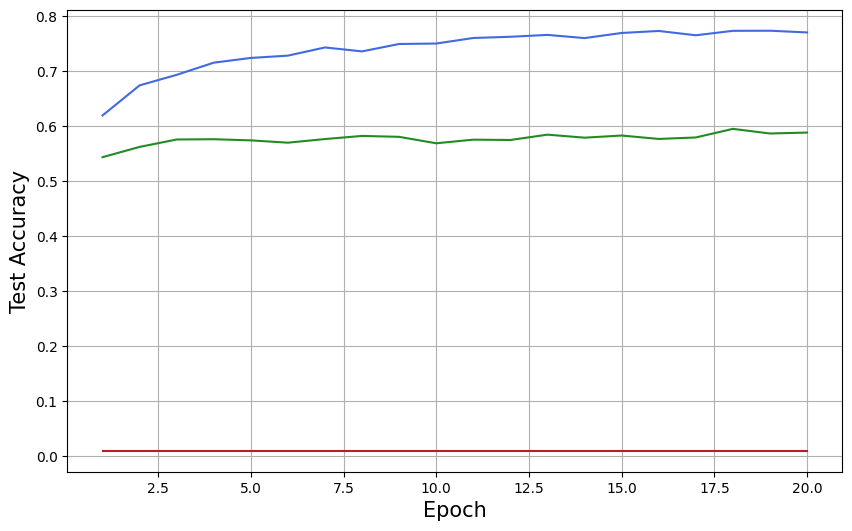} & 
        \includegraphics[width=\imagewidth]{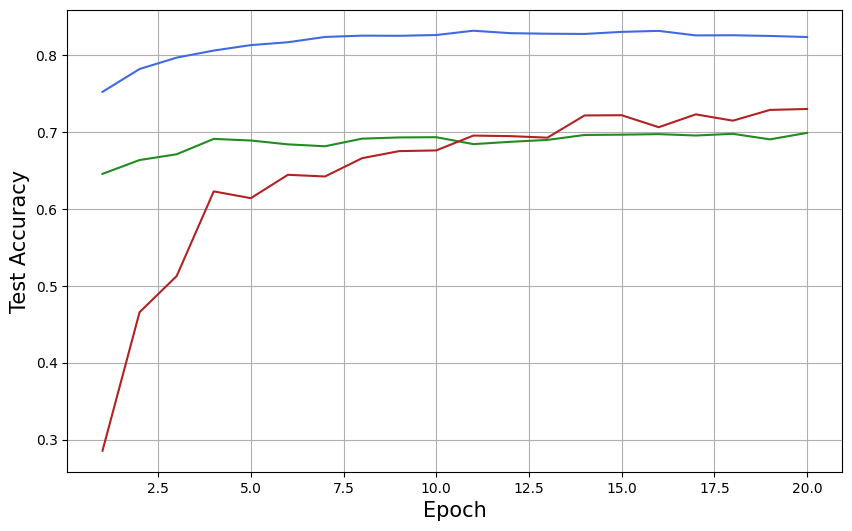} & 
        \includegraphics[width=\imagewidth]{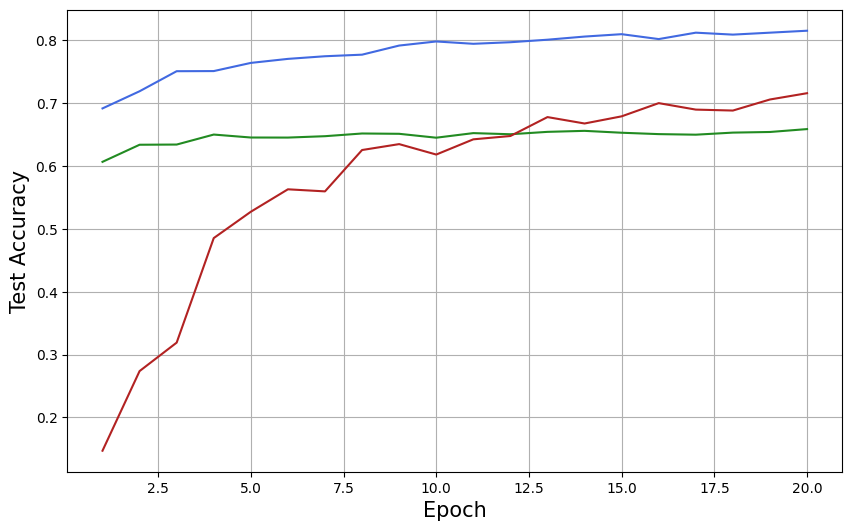} \\
        \rule{0pt}{10pt}\\
    
        \textbf{Dogs} & 
        \includegraphics[width=\imagewidth]{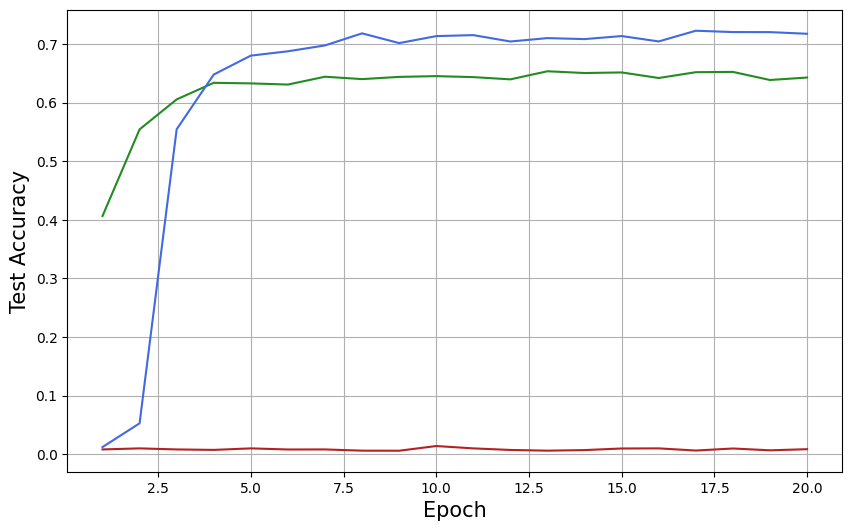} & 
        \includegraphics[width=\imagewidth]{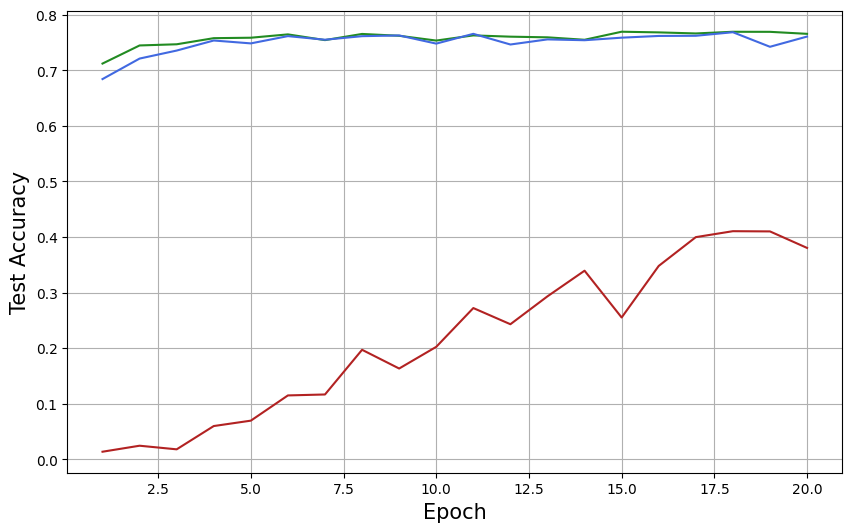} & 
        \includegraphics[width=\imagewidth]{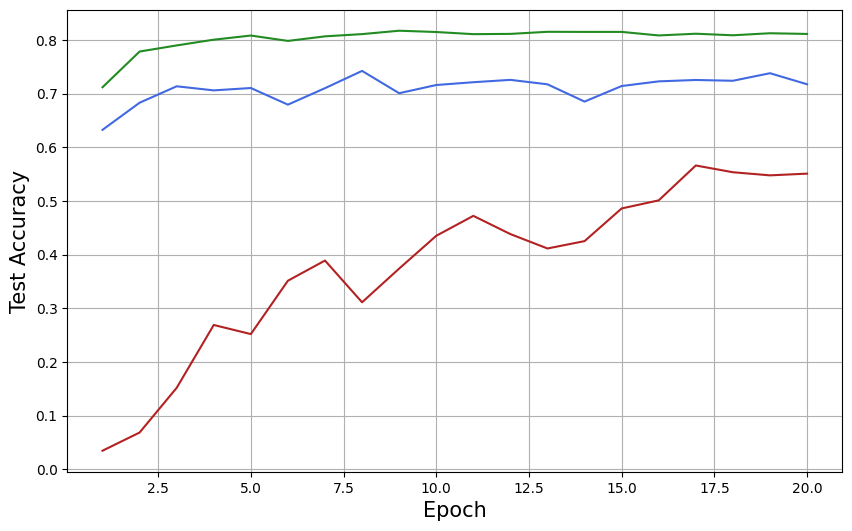} \\
        \rule{0pt}{10pt}\\
    
        \textbf{Pets} & 
        \includegraphics[width=\imagewidth]{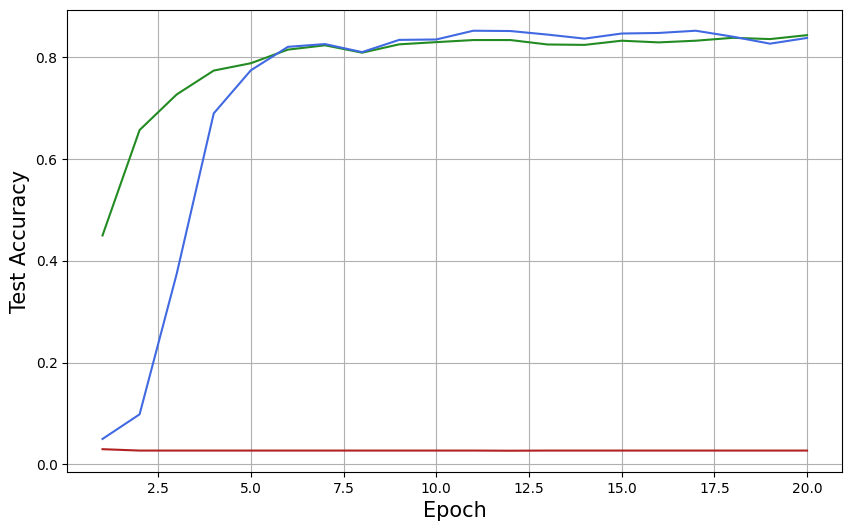} & 
        \includegraphics[width=\imagewidth]{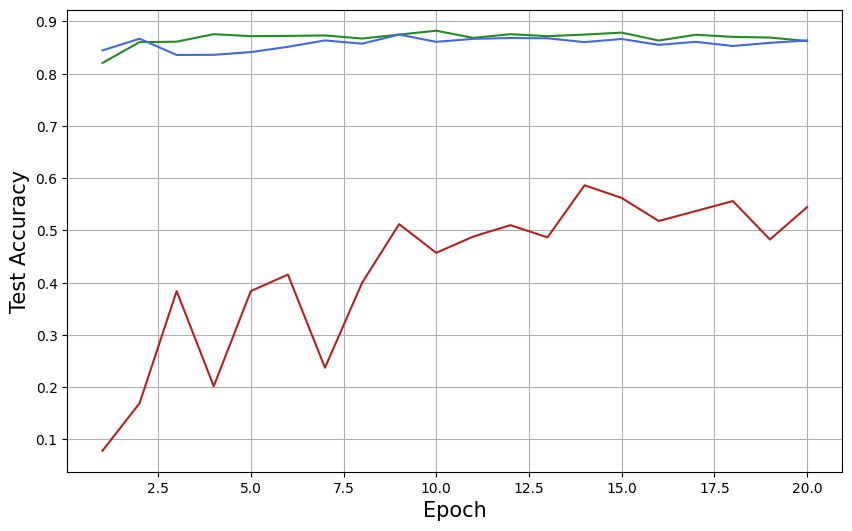} & 
        \includegraphics[width=\imagewidth]{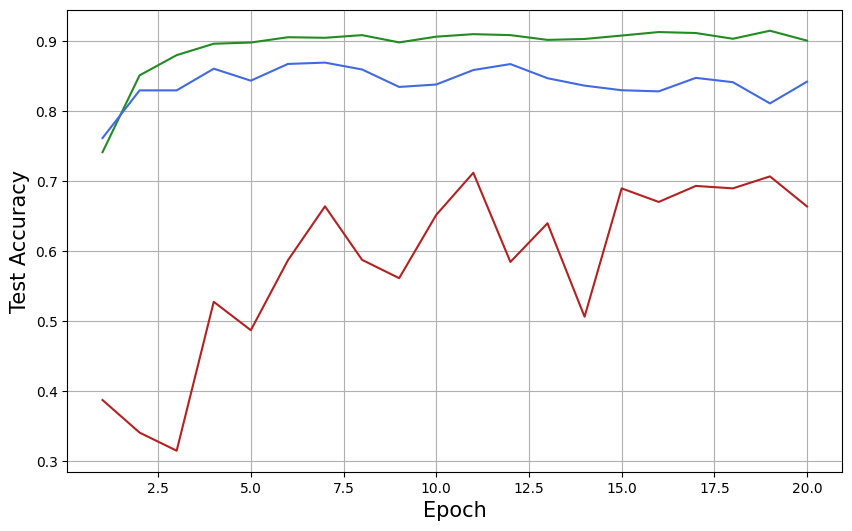} \\
        \rule{0pt}{10pt}\\
    
        \textbf{Caltech} & 
        \includegraphics[width=\imagewidth]{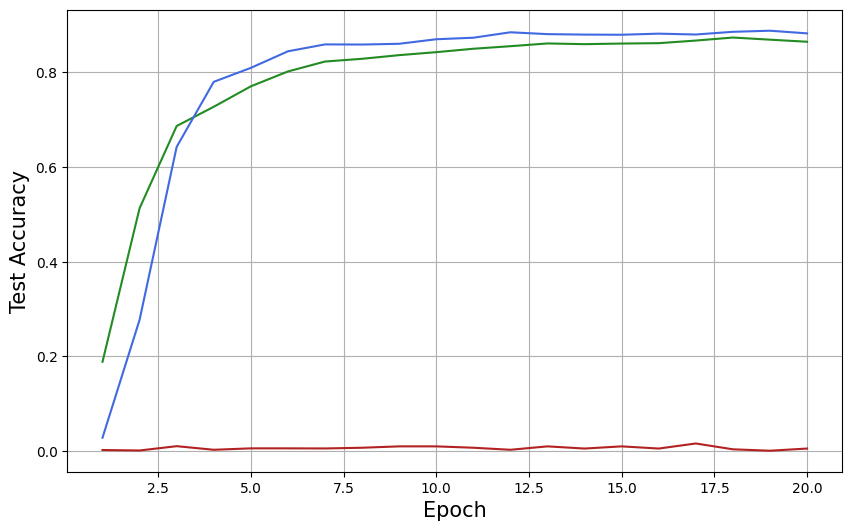} & 
        \includegraphics[width=\imagewidth]{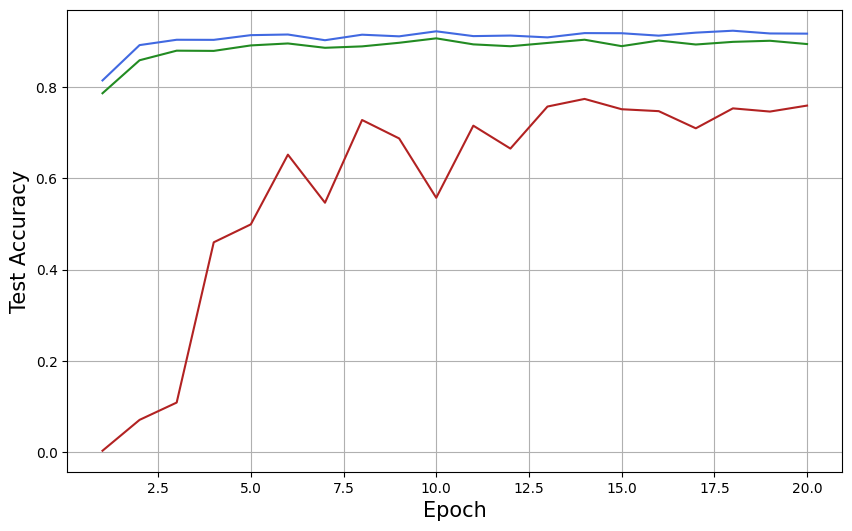} & 
        \includegraphics[width=\imagewidth]{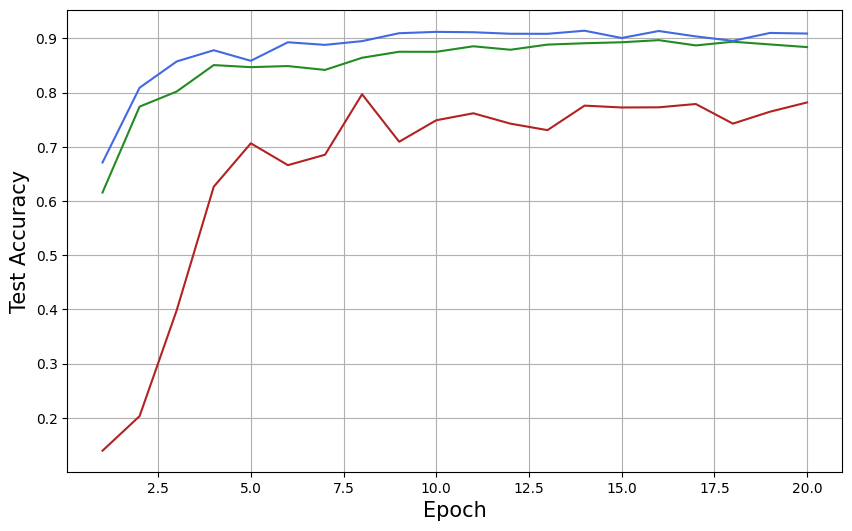} \\
        \rule{0pt}{10pt}\\
    
        \textbf{Flowers} & 
        \includegraphics[width=\imagewidth]{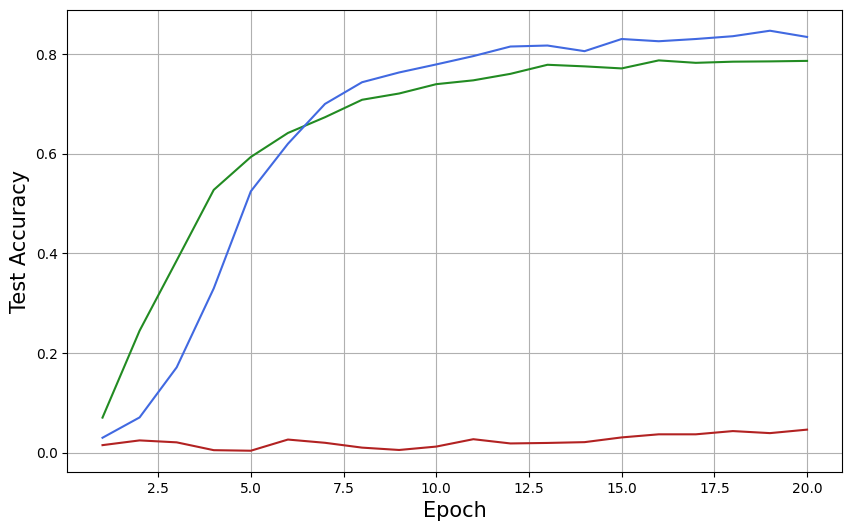} & 
        \includegraphics[width=\imagewidth]{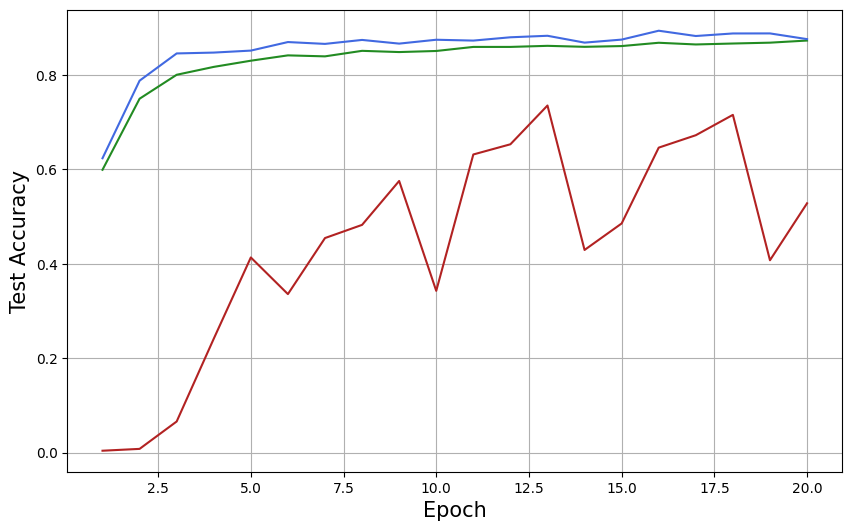} & 
        \includegraphics[width=\imagewidth]{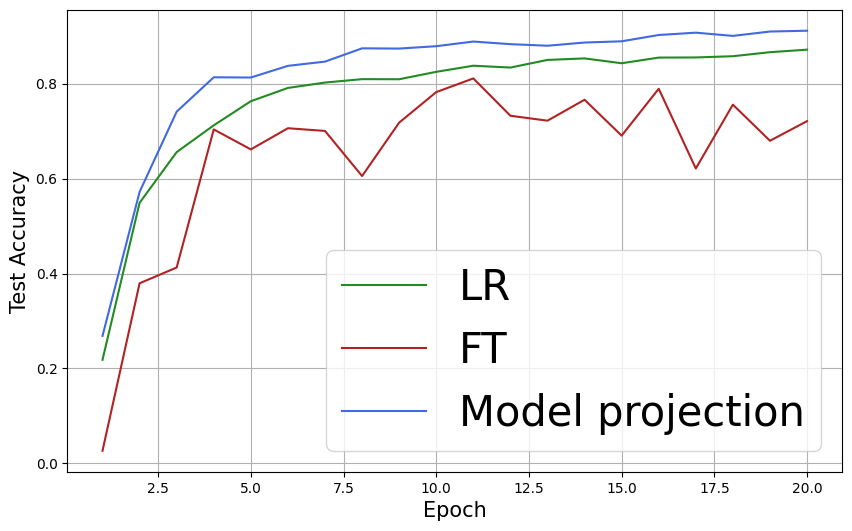} \\
    \end{tabular}
\end{figure}
\newpage
\begin{figure}[h!]
    \centering
    \caption{Results from the second set of experiments. The orange lines show the performance of 2-step fine tuning. The pink lines show the performance of 2-step fine tuning using projection in the first step, and the light blue lines represent the performance of the 2-steps using projection in both steps. Each column corresponds to a particular convolutional base, organized from oldest to most recent. Each row shows the results on a particular dataset, organized from largest to smallest. In all charts, the x-axis is epochs, and the y-axis is test accuracy.}
    \label{fig:2ndExp}
    \setlength{\imagewidth}{0.25\textwidth}

    \begin{tabular}{l c c c} 
        & \textbf{VGG16} & \textbf{ResNet50} & \textbf{DenseNet121} \\
        \rule{0pt}{15pt}\\
        
        \textbf{Food} & 
        \includegraphics[width=\imagewidth]{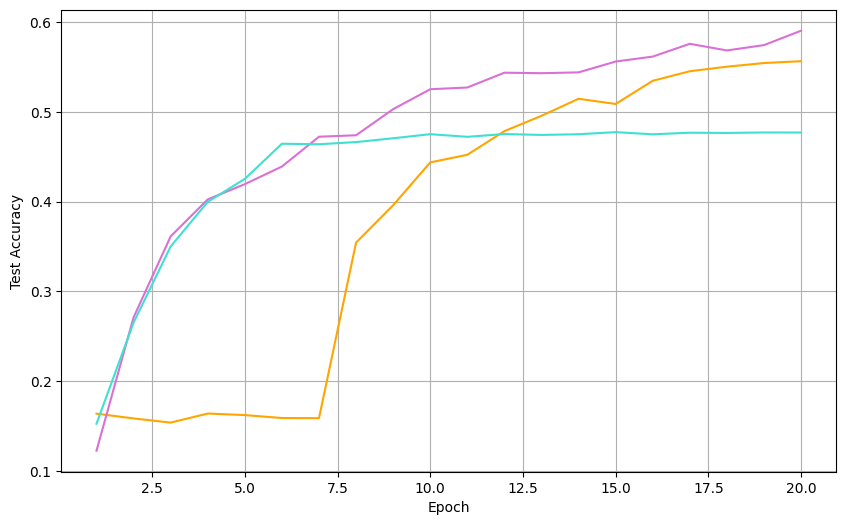} & 
        \includegraphics[width=\imagewidth]{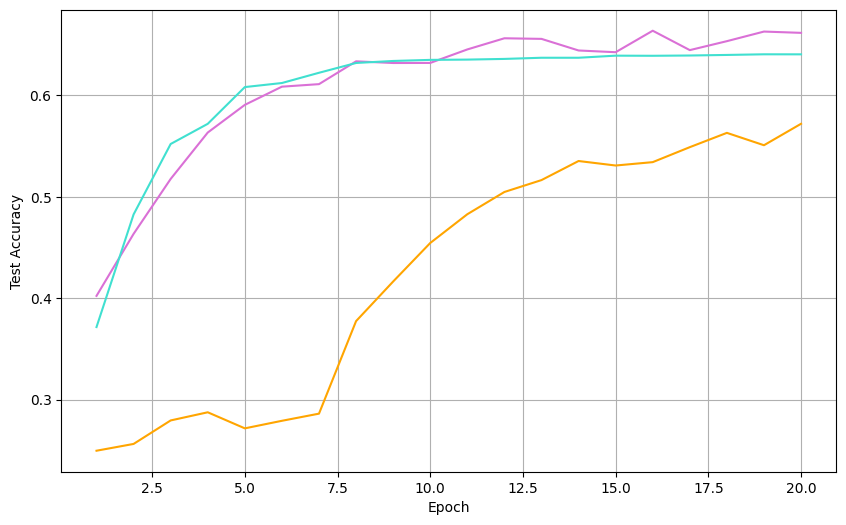} & 
        \includegraphics[width=\imagewidth]{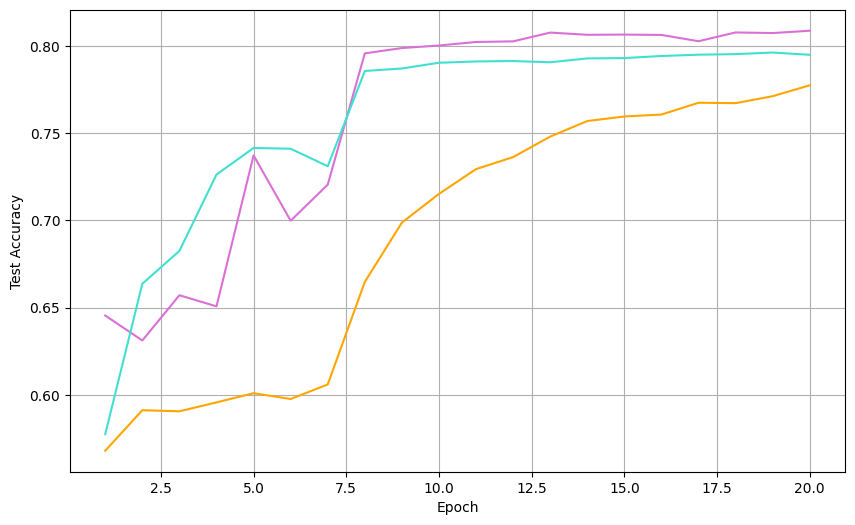} \\
        \rule{0pt}{10pt}\\
    
        \textbf{CIFAR 10} & 
        \includegraphics[width=\imagewidth]{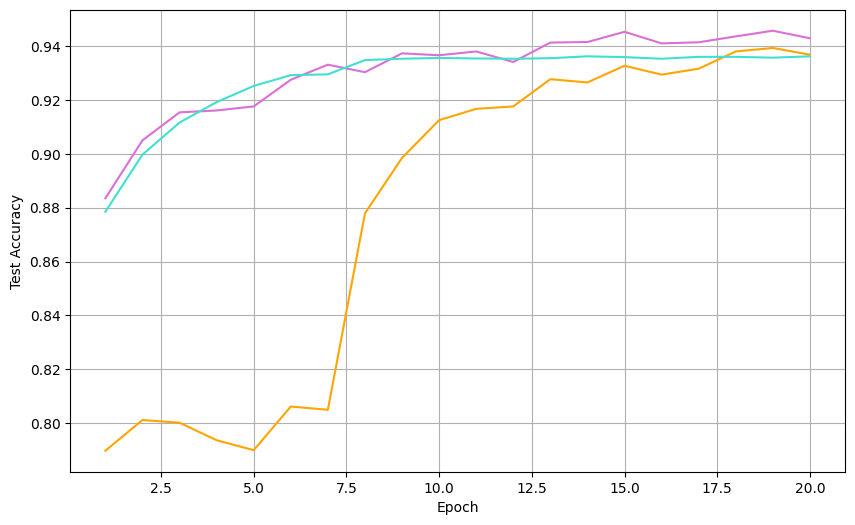} & 
        \includegraphics[width=\imagewidth]{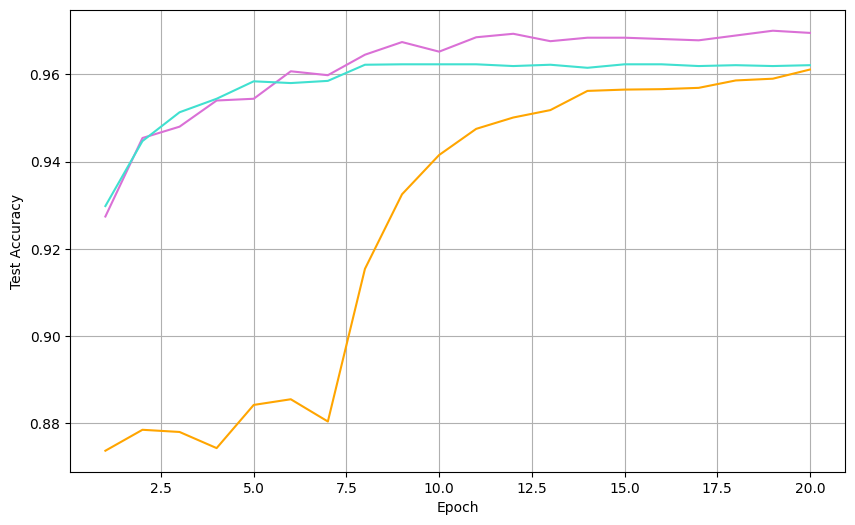} & 
        \includegraphics[width=\imagewidth]{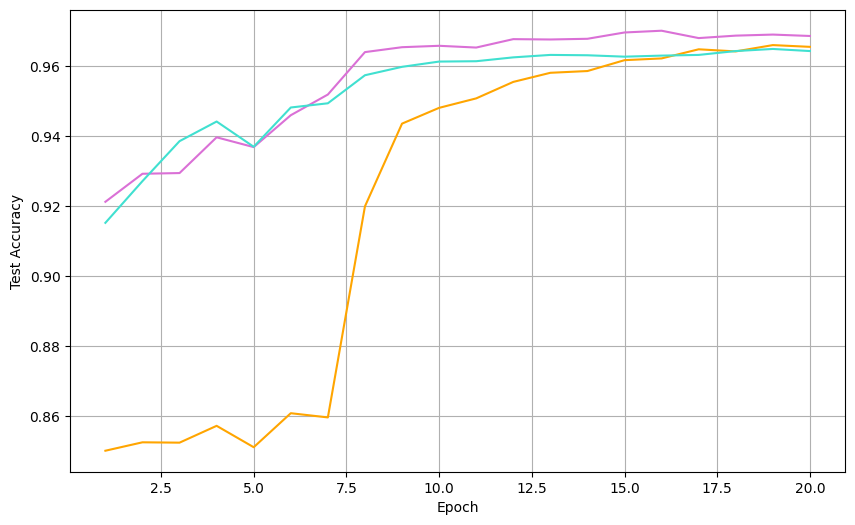} \\
        \rule{0pt}{10pt}\\
    
        \textbf{CIFAR 100} & 
        \includegraphics[width=\imagewidth]{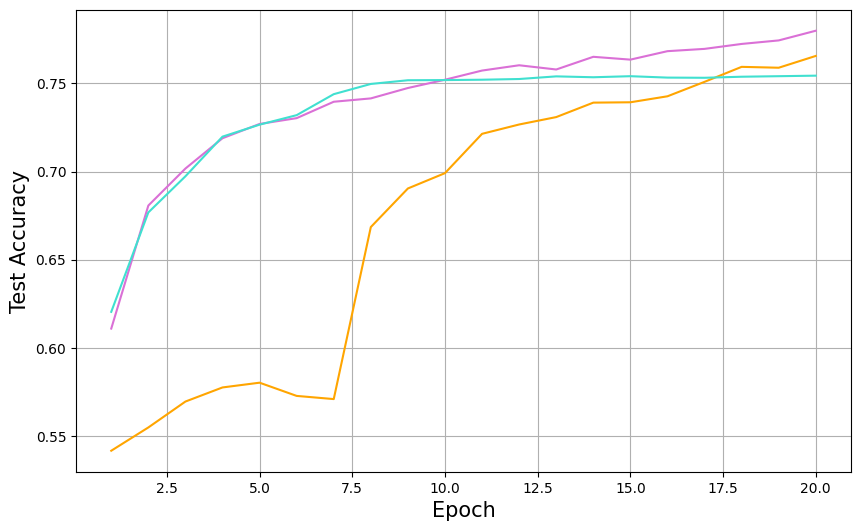} & 
        \includegraphics[width=\imagewidth]{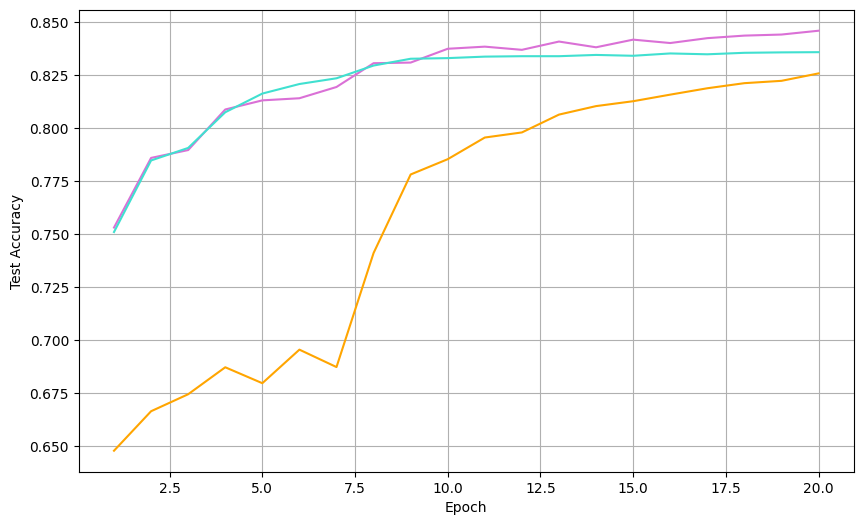} & 
        \includegraphics[width=\imagewidth]{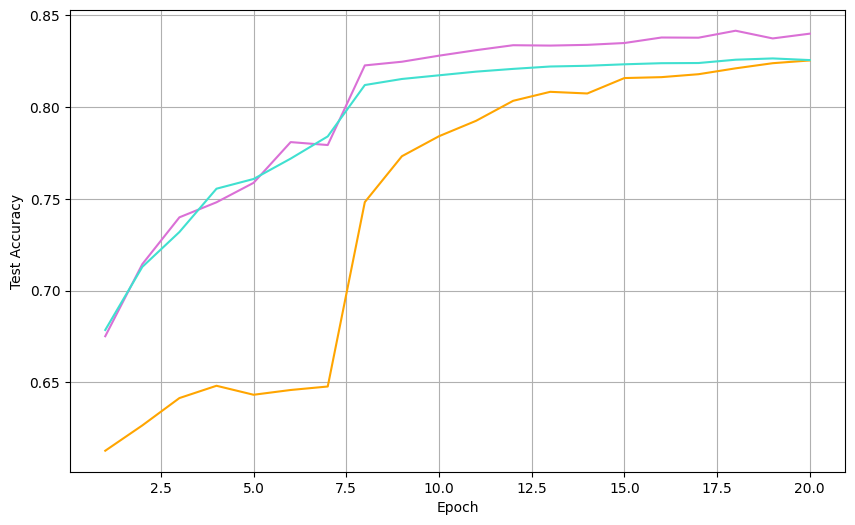} \\
        \rule{0pt}{10pt}\\
    
        \textbf{Dogs} & 
        \includegraphics[width=\imagewidth]{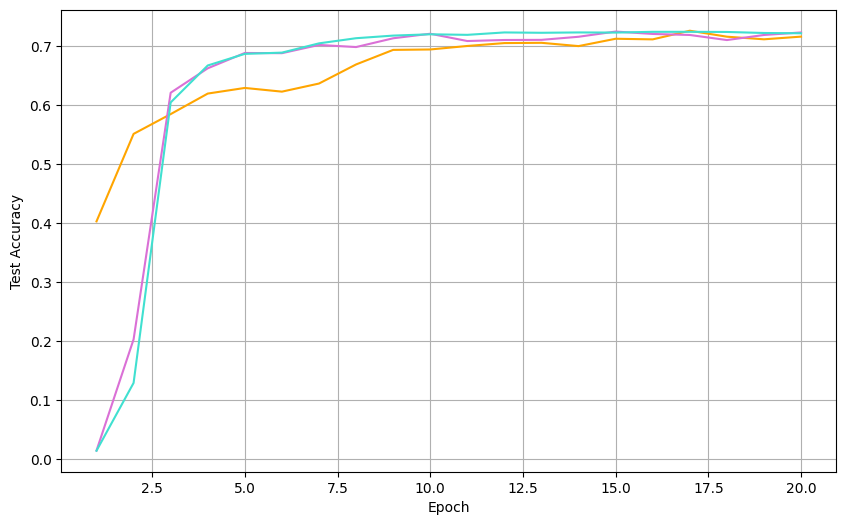} & 
        \includegraphics[width=\imagewidth]{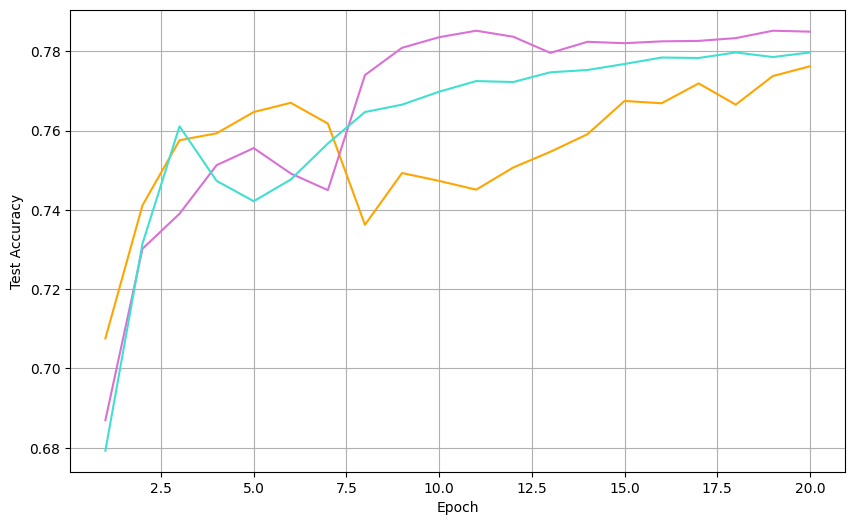} & 
        \includegraphics[width=\imagewidth]{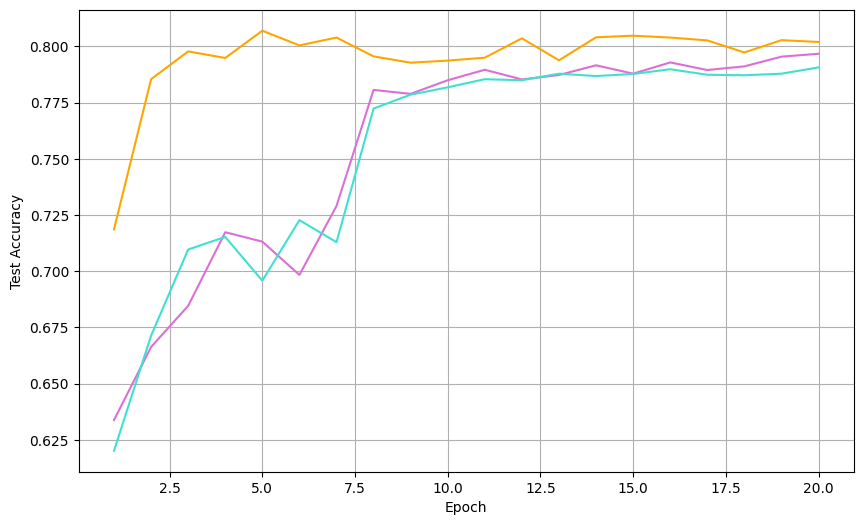} \\
        \rule{0pt}{10pt}\\
    
        \textbf{Pets} & 
        \includegraphics[width=\imagewidth]{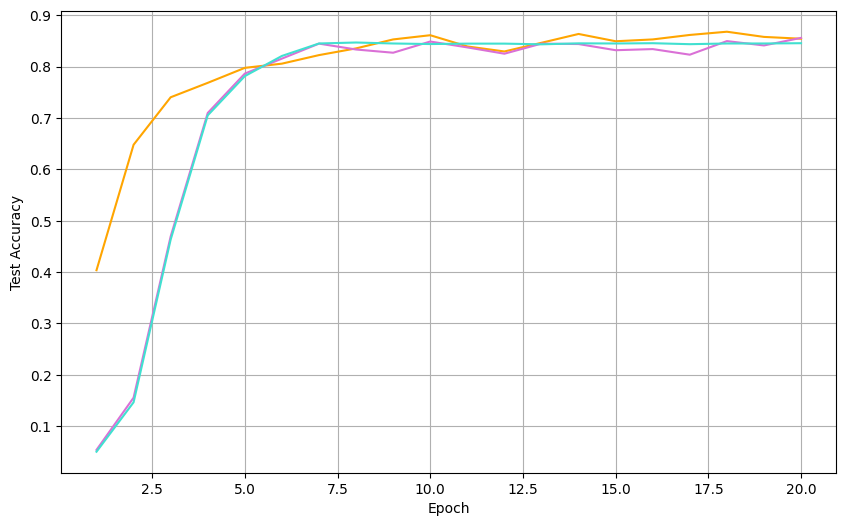} & 
        \includegraphics[width=\imagewidth]{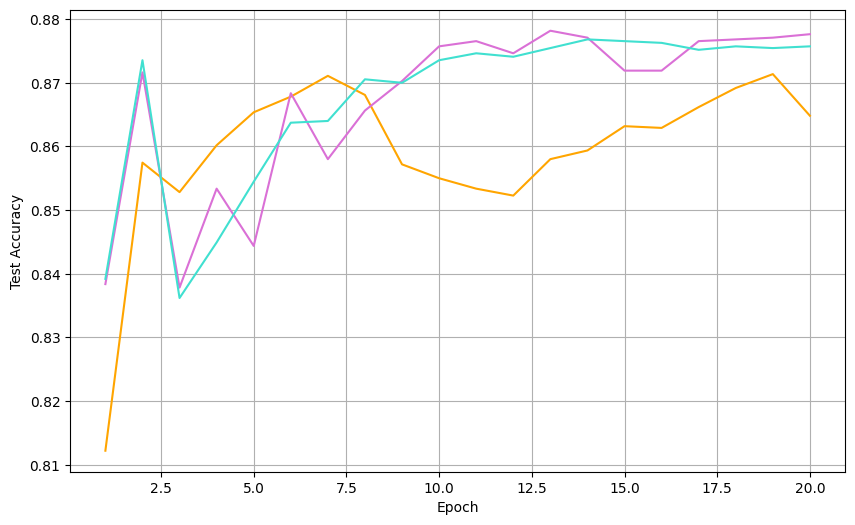} & 
        \includegraphics[width=\imagewidth]{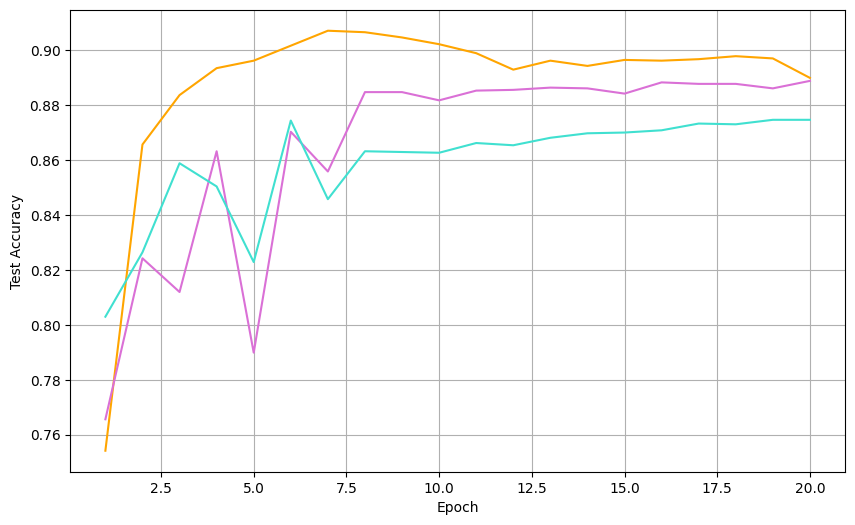} \\
        \rule{0pt}{10pt}\\
    
        \textbf{Caltech} & 
        \includegraphics[width=\imagewidth]{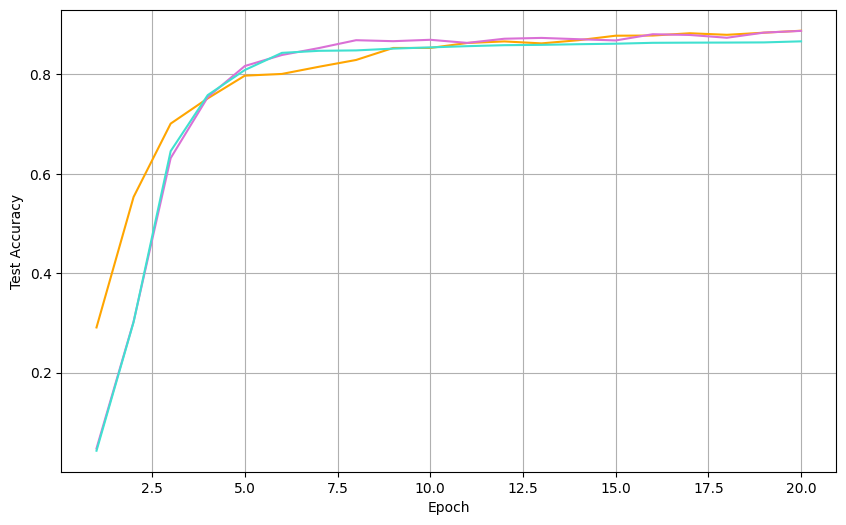} & 
        \includegraphics[width=\imagewidth]{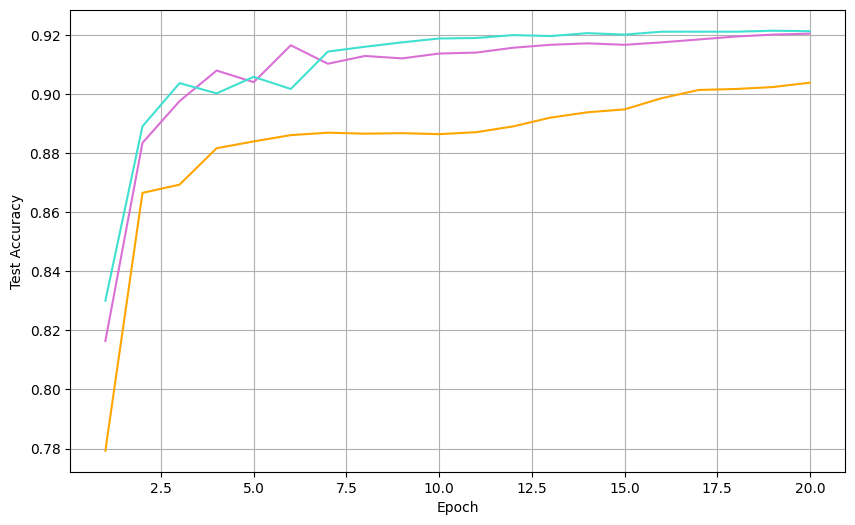} & 
        \includegraphics[width=\imagewidth]{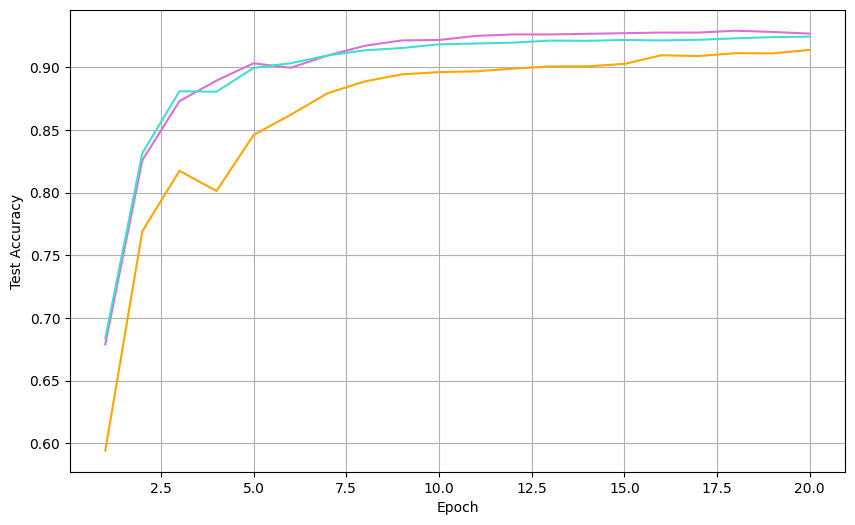} \\
        \rule{0pt}{10pt}\\
    
        \textbf{Flowers} & 
        \includegraphics[width=\imagewidth]{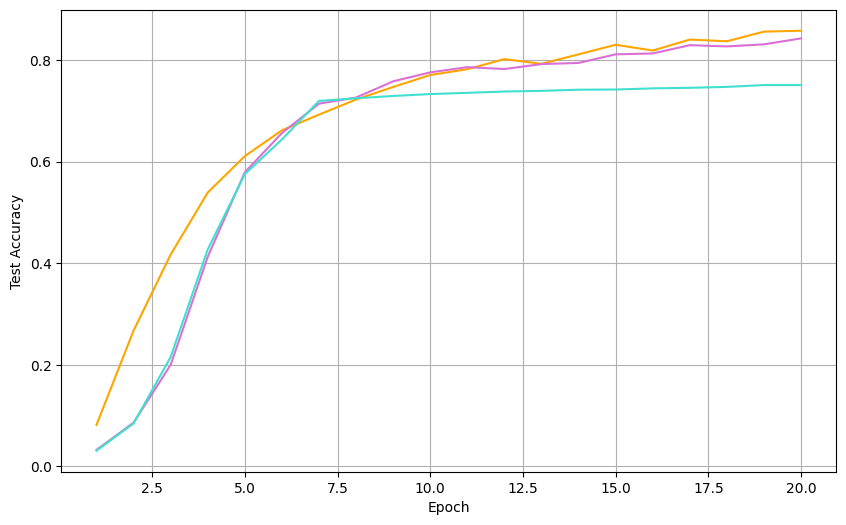} & 
        \includegraphics[width=\imagewidth]{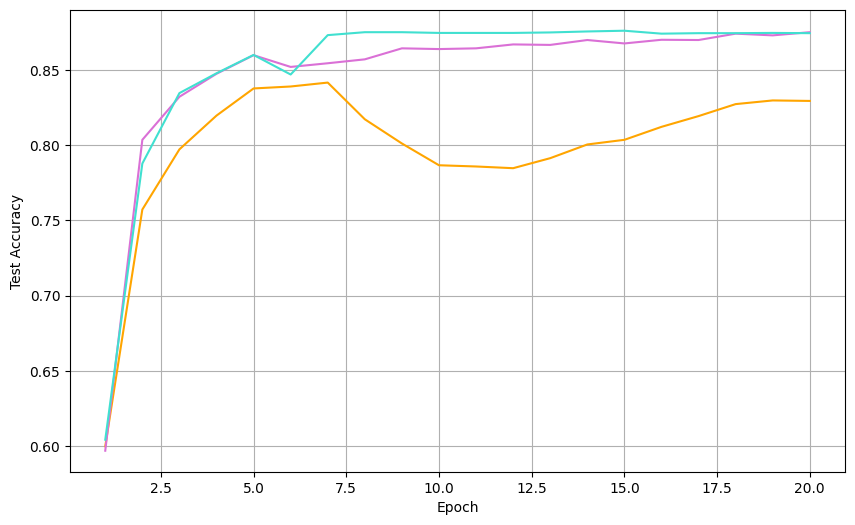} & 
        \includegraphics[width=\imagewidth]{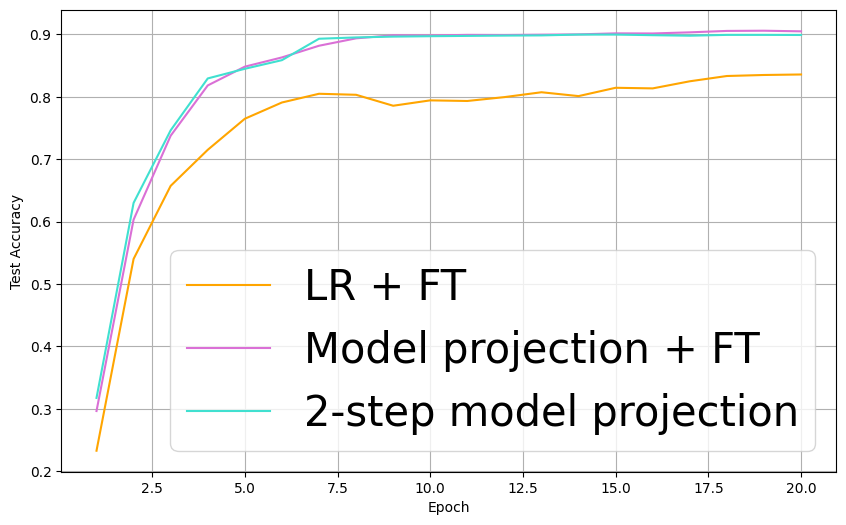} \\
    \end{tabular}
\end{figure}

\end{document}